\documentclass[journal]{IEEEtai}

\usepackage[colorlinks,urlcolor=blue,linkcolor=blue,citecolor=blue]{hyperref}
\usepackage{color,array}
\usepackage{url}            % simple URL typesetting
\usepackage{booktabs}       % professional-quality tables
\usepackage{amsfonts}       % blackboard math symbols
\usepackage{nicefrac}       % compact symbols for 1/2, etc.
\usepackage{microtype}      % microtypography
\usepackage{xcolor}         % colors
\usepackage{algorithm, algorithmicx, algpseudocode}
\usepackage{amsmath}
\usepackage{amssymb}
\usepackage{amsthm}
\usepackage{graphics,graphicx}
\usepackage{bbding}
\usepackage{multirow,multicol}
\usepackage{authblk}
\usepackage[caption=false]{subfig}

\renewcommand{\baselinestretch}{1.1}

\newcommand{\R}{\mathbb{R}}

\newcommand{\bb}{\textbf{b}}
\newcommand{\bW}{\textbf{W}}

\newcommand{\bx}{\bar{x}}

\newcommand{\bz}{\bar{z}}

\newcommand{\pr}{\mathbf{Prob}}

\newcommand{\T}{\intercal}

\newtheorem{thm}{Theorem}

\newtheorem{proposition}[thm]{Proposition}

\begin{document}

\title{Multi-layer Perceptron Trainability Explained via Variability}

\author{Yueyao Yu,
Yin Zhang
\thanks{This paragraph of the first footnote will contain the date on which you submitted your paper for review. It will also contain support information, including sponsor and financial support acknowledgment. For example, ``This work was supported in part by the U.S. Department of Commerce under Grant BS123456.'' }
\thanks{Y. Yu is with the School of Science and
	Engineering, The Chinese University of Hong Kong, Shenzhen, 518172,
	China  and with Shenzhen
	Research Institute of Big Data, China (e-mail: yueyaoyu@link.cuhk.edu.cn).}
\thanks{Y. Zhang is with the School of Data Science, The Chinese University of HongKong, ShenZhen, 518172,
	China (e-mail: yinzhang@cuhk.edu.cn).}

% \thanks{This paragraph will include the Associate Editor who handled your paper.}
}

% \markboth{Journal of IEEE Transactions on Artificial Intelligence, Vol. 00, No. 0, Month 2020}
% {Journals of IEEE Transactions on Artificial Intelligence}

\maketitle

\begin{abstract}
Despite the tremendous successes of deep neural networks (DNNs) in various applications, many fundamental aspects of deep learning remain incompletely understood, including DNN trainability. In a trainability study, one aims to discern what makes one DNN model easier to train than another under comparable conditions. In particular, our study focuses on multi-layer perceptron (MLP) models equipped with the same number of parameters. We introduce a new notion called variability to help explain the benefits of deep learning and the difficulties in training very deep MLPs. Simply put, variability of a neural network represents the richness of landscape patterns in the data space with respect to well-scaled random weights. We empirically show that variability is positively correlated to the number of activations and negatively correlated to a phenomenon called “Collapse to Constant”, which is related but not identical to the well-known vanishing gradient phenomenon. Experiments on a small stylized model problem confirm that variability can indeed accurately predict MLP trainability. In addition, we demonstrate that, as an activation function in MLP models, the absolute value function can offer better variability than the popular ReLU function can.
\end{abstract}

\begin{IEEEImpStatement}
The use of deep neural networks (DNNs) has been driving the recent advances in artificial intelligence, though our understanding of DNNs remains deficient. In this work we study the trainability issue to understand what makes DNN models difficult or easy to train when the number of model parameters is fixed. We have identified and empirically studied a property called variability that demonstrably affects the trainability of a primary class of DNNs (called multilayer perceptrons). Our results provide a new angle to study the issue of DNN trainability and can potentially help design more efficient DNNs that maintain a high level of performance without demanding excessive amounts of computing powers and energy.
\end{IEEEImpStatement}

%In this paper, we investigate the issue of DNN trainability and focus on the role of the property called variability in affecting the trainability of a primary class of DNNs, the multilayer perceptrons. Our findings shed new light on DNN trainability and have the potential to guide the design of more efficient DNNs that maintain high performance while minimizing computational and energy demands.

\begin{IEEEkeywords}
Deep neural network, multi-layer perceptron, trainability, variability, collapse to constant, absolute-value activation
\end{IEEEkeywords}

\section{Introduction}

\IEEEPARstart{D}{eep} neural networks (DNNs)  have achieved remarkable success in various fields, but many fundamental issues are still not fully understood, including the trainability of DNNs.  Recently, researchers have explored the trainability of DNNs in the infinite-width limit using mean-field theory or neural tangent kernel methods, e.g.~\cite{schoenholz2017deep, jacot2018neural, xiao2020disentangling}. In this paper, we propose a new approach to studying the trainability of MLP models under the setting where the total number of model parameters is fixed. 

The purpose of this work is to gain valuable insights into behaviors of DNNs.  Our  contributions are mainly conceptual, consisting of the following aspects.

\begin{enumerate}
	\item 
	We introduce the concept of variability and investigate two different measurements, providing a novel perspective on understanding the advantages of deep learning and the difficulties associated with training. Specifically, for MLPs with a fixed number of parameters, we show that variability initially rises and then falls as MLP depth grows.
	
	\item
	We show that the initial increase in variability coincides with the increase of the activation ratio, while the subsequent decrease is due to a phenomenon called Collapse to Constant (C2C) that is distinct from gradient vanishing phenomenon. We explain the similarities and differences of the two phenomena through their characterization matrices.

	\item 
	Experiments on a stylized model problem provide strong evidence suggesting that variability is a critical indicator for  training performance on deep MLPs.
	
	\item
	We show that the absolute-value function (ABS), when used as an activation function in MLPs, generally provides higher variability than the popular ReLU function. Indeed, experiments confirm that ABS generally yields better training results than ReLU does.
\end{enumerate}

\section{MLP: notations and settings}

We first introduce notations and neural network settings used throughout the article.
\subsection{Notations} 
We consider MLP models comprising of an input layer, an output layer, and $L+1$ hidden layers for $L \ge 0$. 
It is constructed from $L$ affine maps represented by a sequence of weight matrices $\{W_k\}$ and bias vectors $\{b_k\}$ of compatible sizes for $k=1,\cdots,L$.  We denote the collections of such weight matrices and bias vectors, respectively, by
\begin{equation*}
\bW = \{W_1,\cdots,W_L\} ~~\mbox{ and }~~ \bb = \{b_1,\cdots,b_L\}.
\end{equation*}
For ease of discussions, we will tacitly assume that all weight matrices $W_i \in \R^{d\times d}$ and all bias vectors $b_i \in \R^d$.  This assumption will have no substantive impact on our conclusions.

At each hidden layer $k \in \{1,\cdots,L\}$, we define
%a \emph{hidden-layer function} $\psi_k(\cdot)$ is applied to the input of the layer, where
\begin{equation}\label{def:Gi}
%\psi_k(\cdot) \equiv 
\psi_k(\cdot,W_k,b_k) := \phi(W_k(\cdot)+b_k):  \R^d \rightarrow\R^d,
\end{equation}
which is the composition of an activation function $\phi(\cdot)$ with the affine function defined by the weight-bias pair $(W_k,b_k)$.  Normally, $\phi$ is a scalar function applied component-wise to vectors. In this paper, we will use three activation functions: Sigmoid for $\phi(t) = 1/(1+e^{-t})$, ReLU for $\phi(t) = \max(0,t)$ and ABS for $\phi(t)=|t|$.  

For convenience, we often drop the dependence of $\psi_k$ on the parameter pair $(W_k,b_k)$ whenever no confusion arises.

We define an {MLP function} to be
%aggregates all $L$ functions in \eqref{def:Gi} to form
\begin{equation}\label{def:Fk}
F_L(\cdot) \equiv F_L(\cdot,\bW,\bb) := (\psi_L\circ\cdots\circ \psi_1)(\cdot): \R^d \rightarrow\R^d,
%\equiv \psi_L(\psi_{k-1}(\cdots \psi_2(\psi_1(x))))),
\end{equation}
which is the composition of $\psi_1$ to $\psi_L$ and parameterized by the aggregated pair $(\bW,\bb)$. For any given parameter pair, the network maps an input $x$ to an output $F_L(x,\bW,\bb)$ that can be computed through the forward propagation: set $s_0 = x$,
\begin{equation}\label{def:recursion}
	z_k = W_ks_{k-1}+b_k, s_k = \phi(z_k), \;\; k = 1,\cdots,L.
\end{equation}
Then, $F_L(x,\bW,\bb)=s_L$ at the end. 

In our notation, subscripts usually are reserved as indices of hidden layers. On the other hand, we use $[v]_i$ to denote the $i$-th element of a vector $v$, and similarly for matrix elements.

\subsection{A normalized MLP setting}
We will examine certain properties of the MLP function $F_L(x,\bW,\bb)$ as the number of layers $L$ increases while the number of total parameters is kept a constant $N_w$. Since the total number of parameters of a neural network is a dominant cost factor in the training and deployment of the network, it is appropriate to study  architectural issues of neural networks under a normalized setting with fixed costs.

To facilitate subsequent experiments in this paper, we add an input layer and an output layer, both of dimension 2, to the $L+1$ hidden layers.  For convenience, we continue to use $F_L(x)$ to denote the extended network which now has become a map from $\R^2$ to $\R^2$. In this case, the total number of parameter is
\begin{equation}\label{eq:Nw}
	N_w = Ld^2+(L+5)d+2.
\end{equation}
Figure~\ref{fig:mlp} shows an MLP example with $L=1$ and $d=10$, leading to $N_w=162$.
\begin{figure}[htb]
	\begin{center}
		\includegraphics[width=.4\textwidth,trim=120 190 110 150,clip]{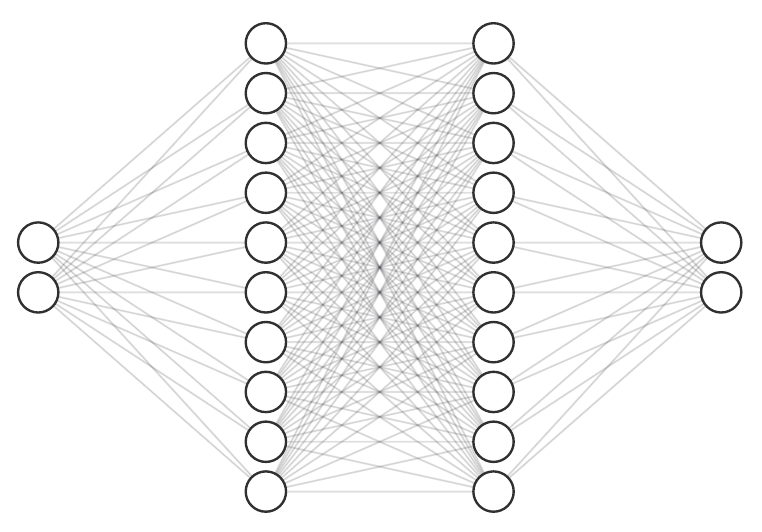}
		
		\caption{An example of $F_L(\cdot)$ where $L=1$ and $d=10$.}
		\label{fig:mlp}
	\end{center}
	%	\vspace{-.5cm}
\end{figure}

\subsubsection{Width $d$ v.s. Depth  $L$}
Solving \eqref{eq:Nw}, we obtain 
\begin{equation}\label{eq:d}
	d = {\left(\sqrt{(L+5)^2+4L(N_w-2)}-(L+5)\right)}/{2L},
\end{equation}
which of course is not necessarily an integer. We will make small adjustments to the $d$-value (by adding or deleting one or two nodes from some hidden layers) to keep the total number of model parameters as close to a prescribed constant as possible.

\subsubsection{Initializing and Scaling}
\label{scale}

It is well-known that the trainability of a model is heavily influenced by its weight initialization strategy. In this study, we adopt the commonly used strategy: to initialize model parameters as random numbers from the standard normal distribution and then scale them.  Specifically, we employ the Xavier initialization for Sigmoid and ABS functions and the Kaiming initialization for ReLU~\cite{glorot2010understanding,he2015delving}. 

From now on, we will assume that the parameters $\bW$ and $\bb$ in the MLP function $F_L(x,\bW,\bb)$ are always initialized and scaled by the above standard initialization schemes.  

%that are selected from normal distributions 
%\begin{equation}\label{W,b-Normal}
%	[W_k]_{ij}, \iid \cN(0,\sigma), \;\;\; [b_k]_i \iid \cN(0,\sqrt{1/d}),
%\end{equation}
%where the variance $\sigma$ for weights is set according to the activation function in use.   

\section{What is variability}
\begin{figure*}[th] 
	\centering  
	%	\captionsetup[subfloat]{labelsep=none,format=plain,labelformat=empty}
	
	\subfloat[Sigmoid~($L=1$)]{
		\begin{minipage}[b]{0.3\textwidth}  
			\includegraphics[width=0.48\textwidth,trim=60 50 50 50,clip]{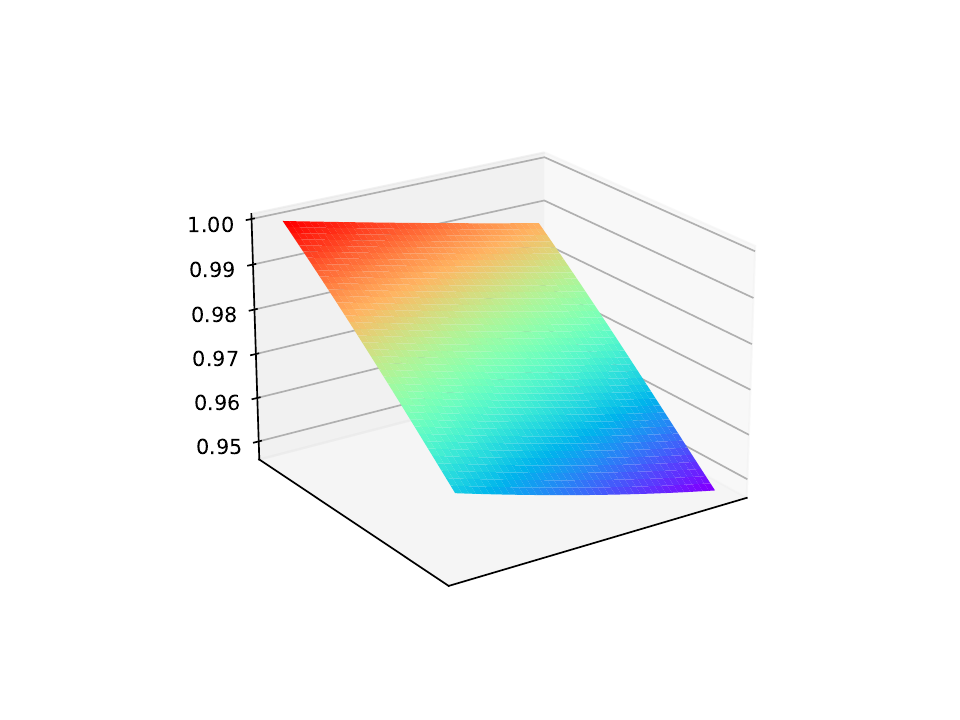}    \includegraphics[width=0.48\textwidth,trim=60 50 50 50,clip]{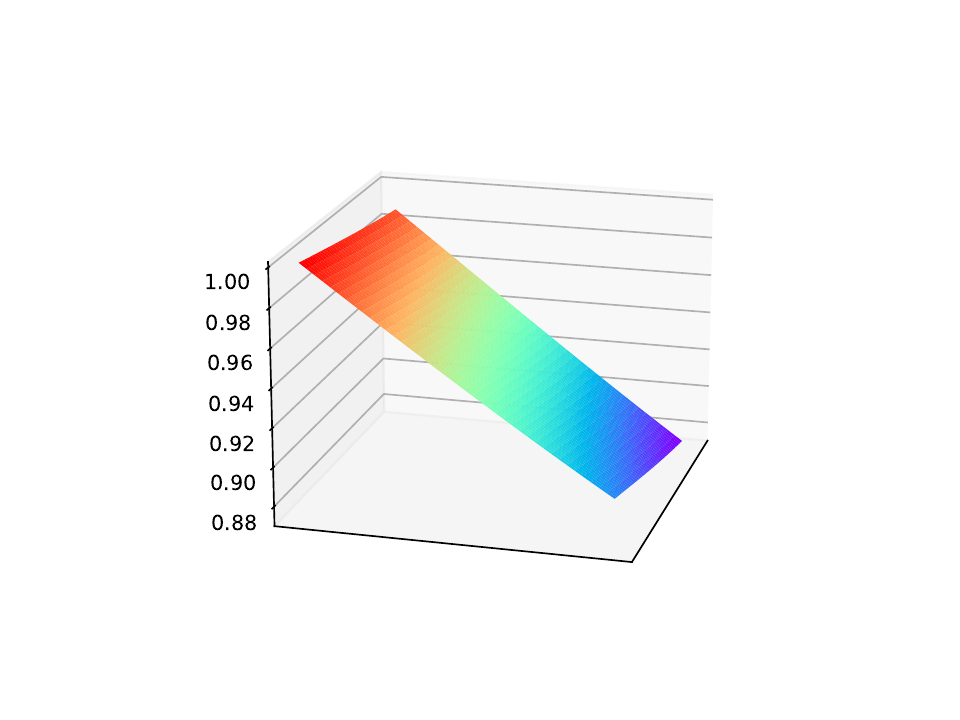} \\
			\includegraphics[width=0.48\textwidth,trim=60 50 50 50,clip]{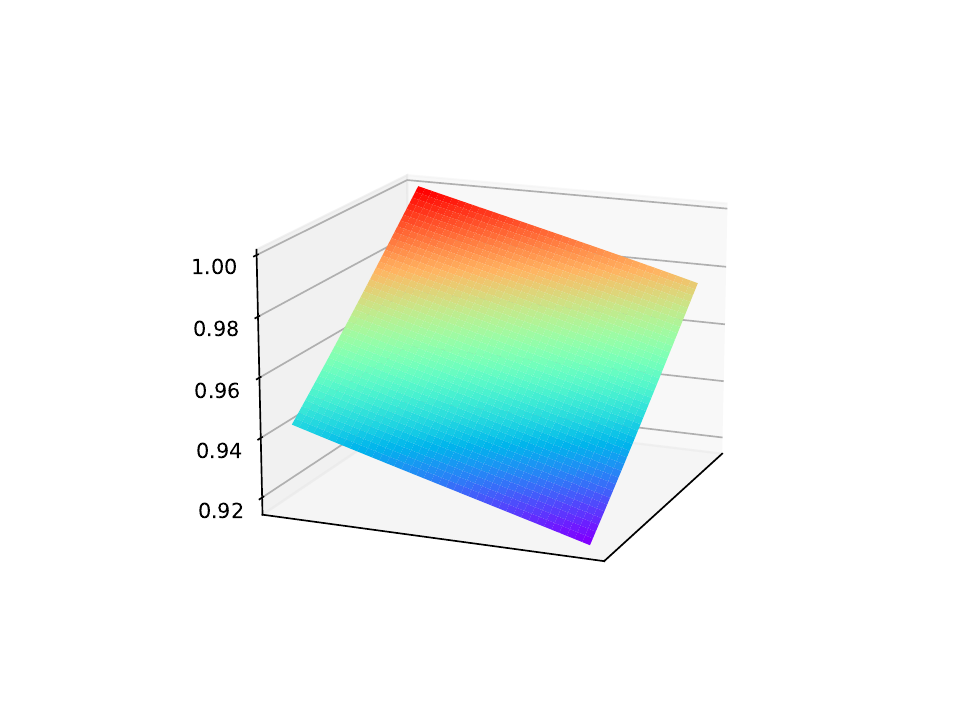}    \includegraphics[width=0.48\textwidth,trim=60 50 50 50,clip]{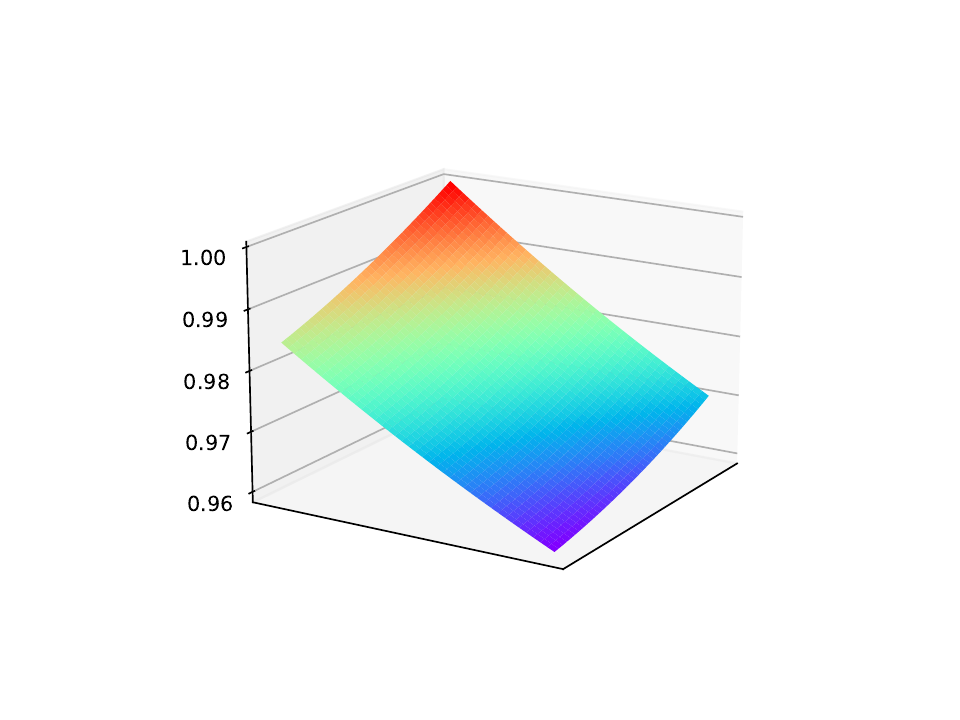} \\ 
	\end{minipage}   }  
	\subfloat[Sigmoid~($L=10$)]{  
		\begin{minipage}[b]{0.3\textwidth}  
			\includegraphics[width=0.48\textwidth,trim=60 50 50 50,clip]{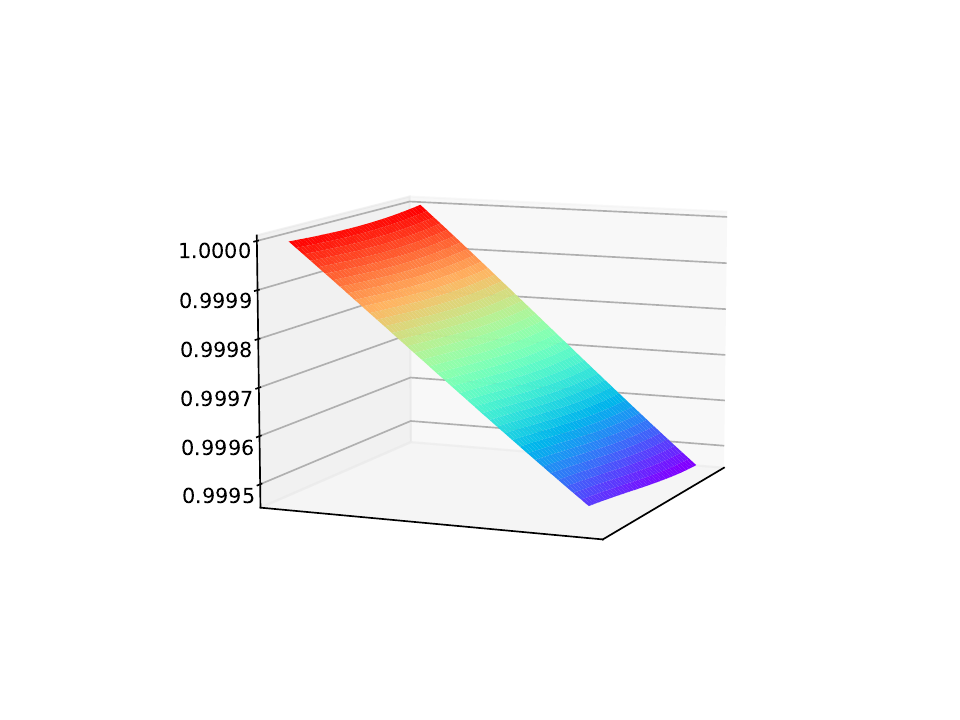}    \includegraphics[width=0.48\textwidth,trim=60 50 50 50,clip]{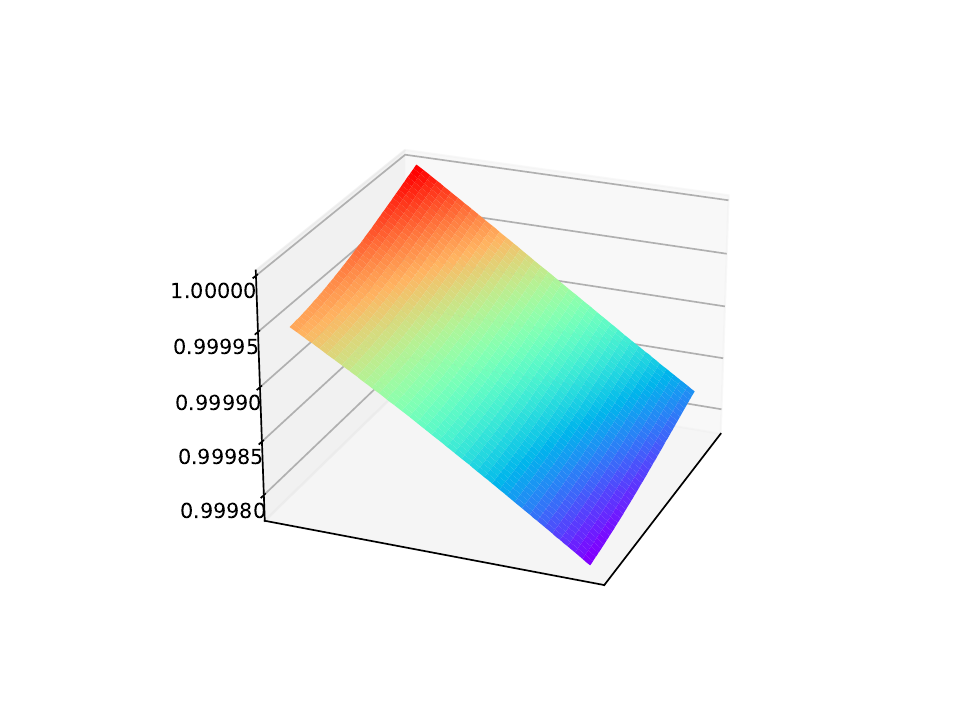} \\
			\includegraphics[width=0.48\textwidth,trim=60 50 50 50,clip]{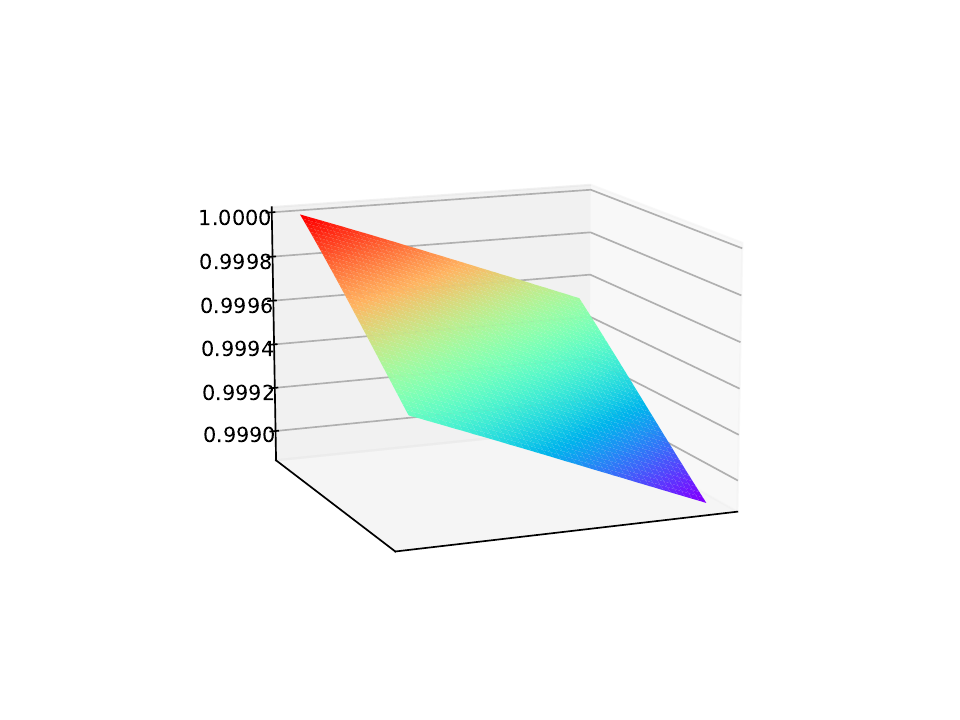}    \includegraphics[width=0.48\textwidth,trim=60 50 50 50,clip]{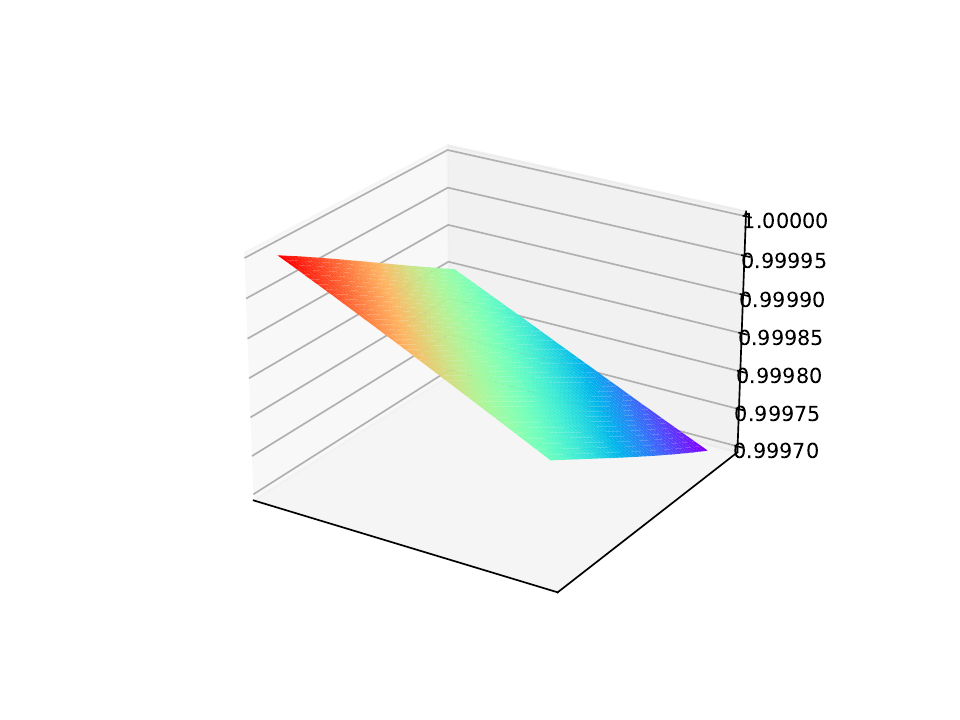} \\ 
	\end{minipage}   }  
	\subfloat[Sigmoid~($L=20$)]{  
		\begin{minipage}[b]{0.3\textwidth}  
			\includegraphics[width=0.46\textwidth,trim=60 50 50 50,clip]{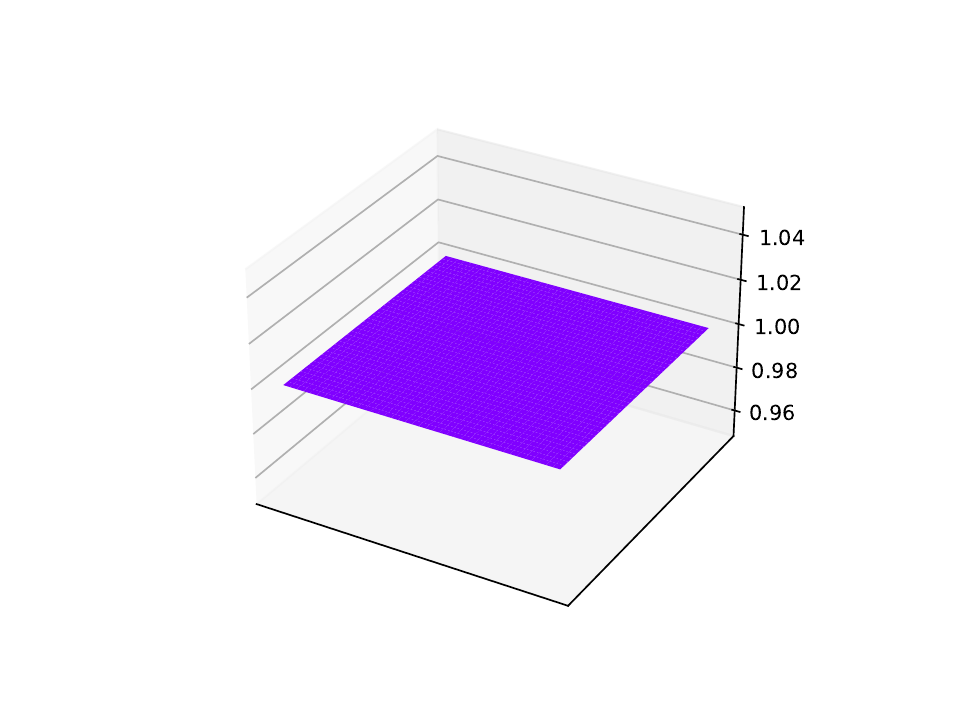}    \includegraphics[width=0.46\textwidth,trim=60 50 50 50,clip]{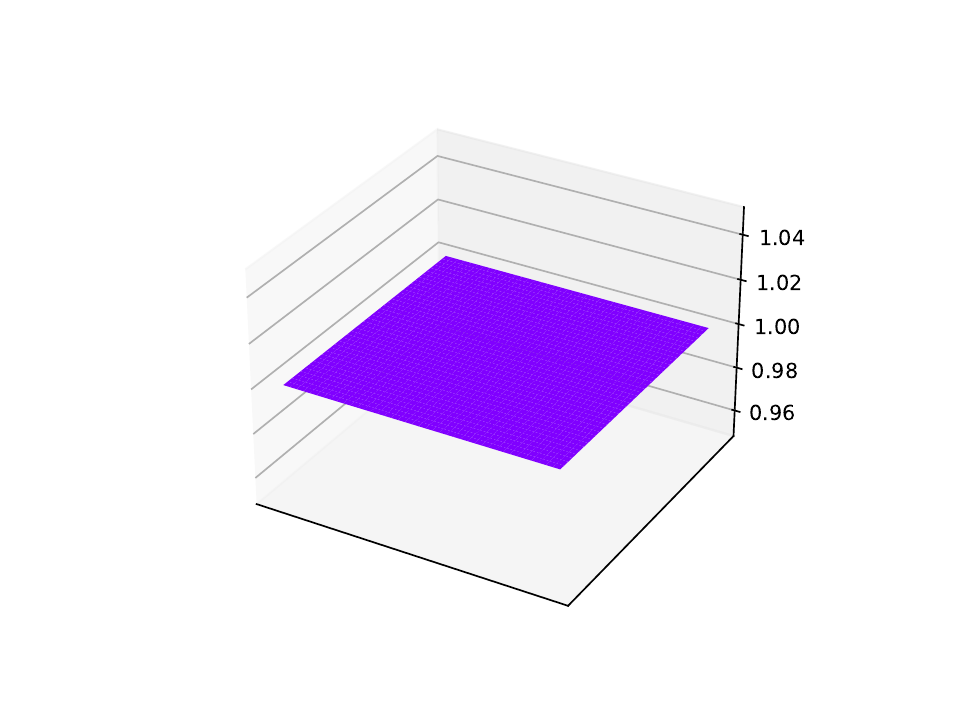} \\
			\includegraphics[width=0.46\textwidth,trim=60 50 50 50,clip]{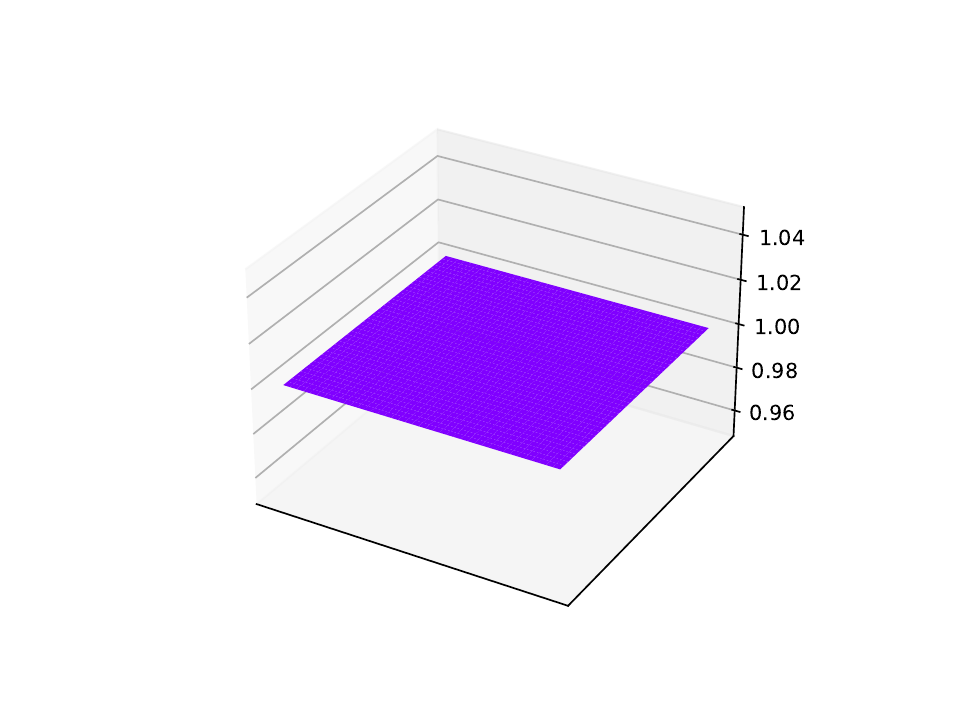}    \includegraphics[width=0.46\textwidth,trim=60 50 50 50,clip]{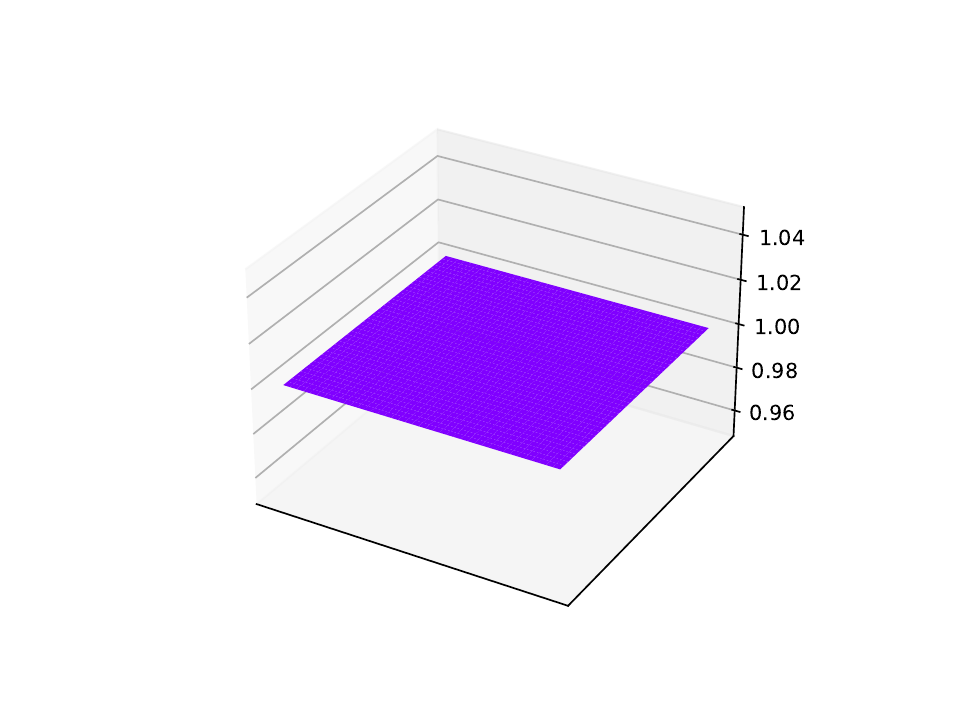} \\ 
	\end{minipage}   }  
	
	\subfloat[ReLU~($L=1$)]{ 
		\begin{minipage}[b]{0.3\textwidth}  
			\includegraphics[width=0.48\textwidth,trim=60 50 50 50,clip]{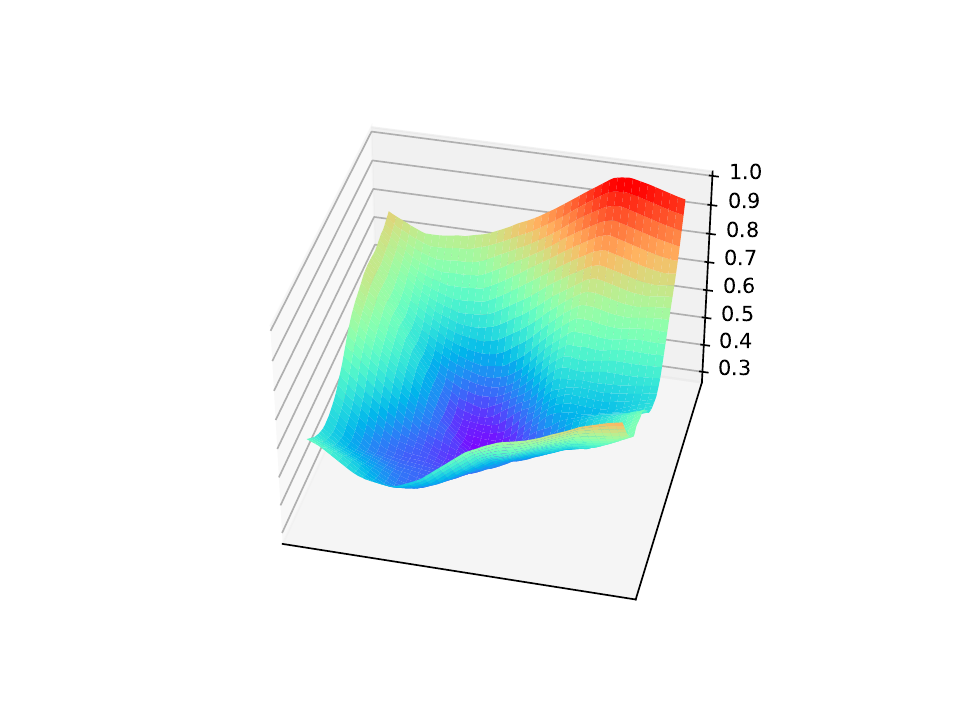}    
			\includegraphics[width=0.48\textwidth,trim=60 50 50 
			50,clip]{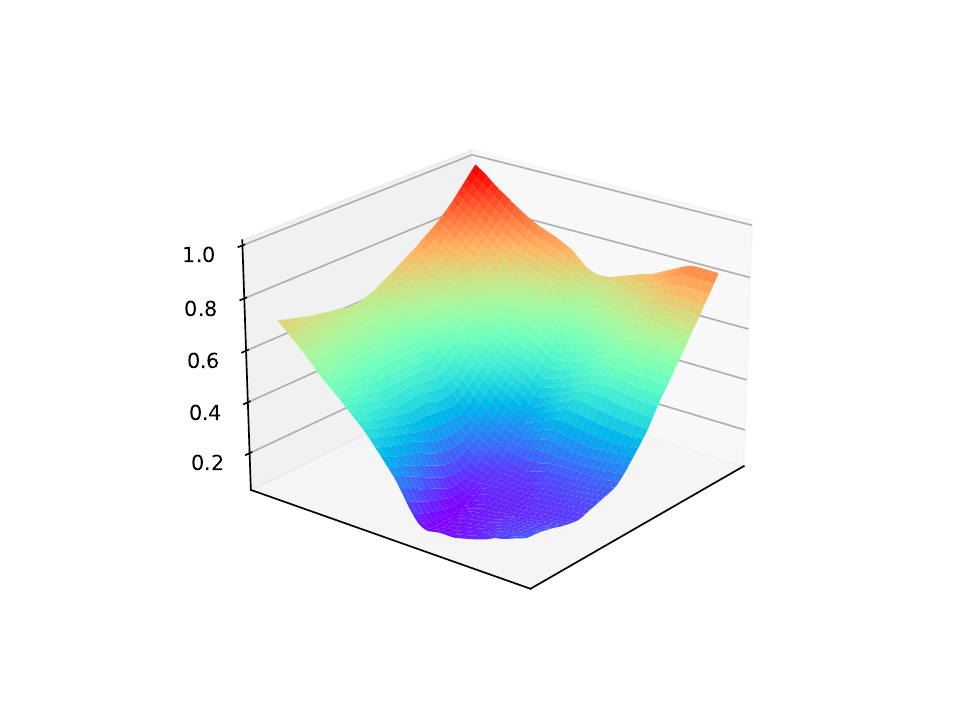} \\
			\includegraphics[width=0.48\textwidth,trim=60 50 50 50,clip]{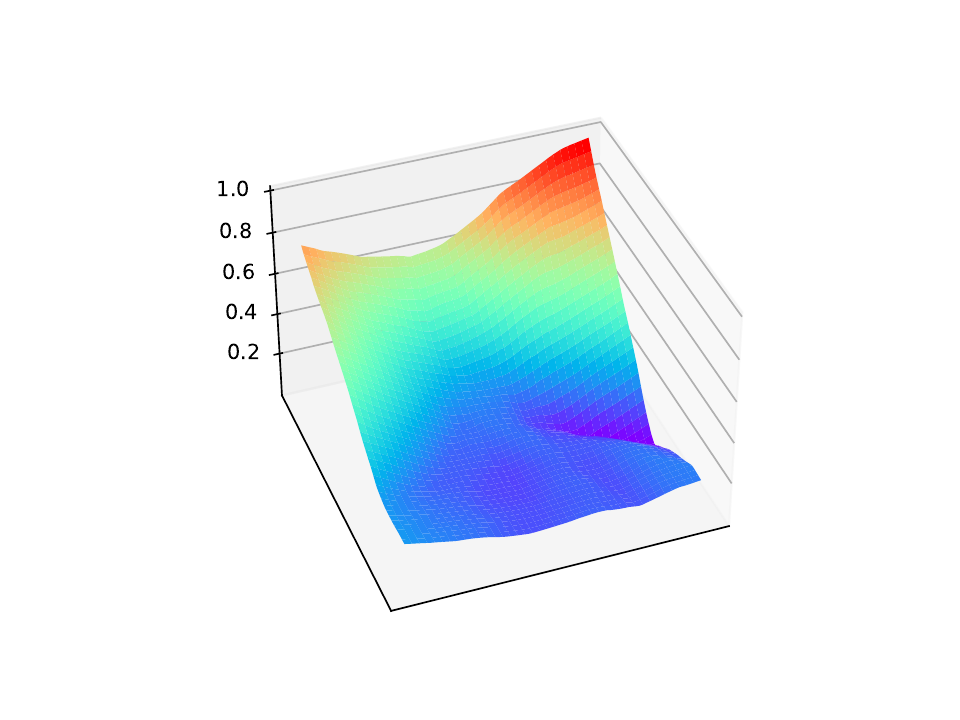}   
			\includegraphics[width=0.48\textwidth,trim=60 50 50 50,clip]{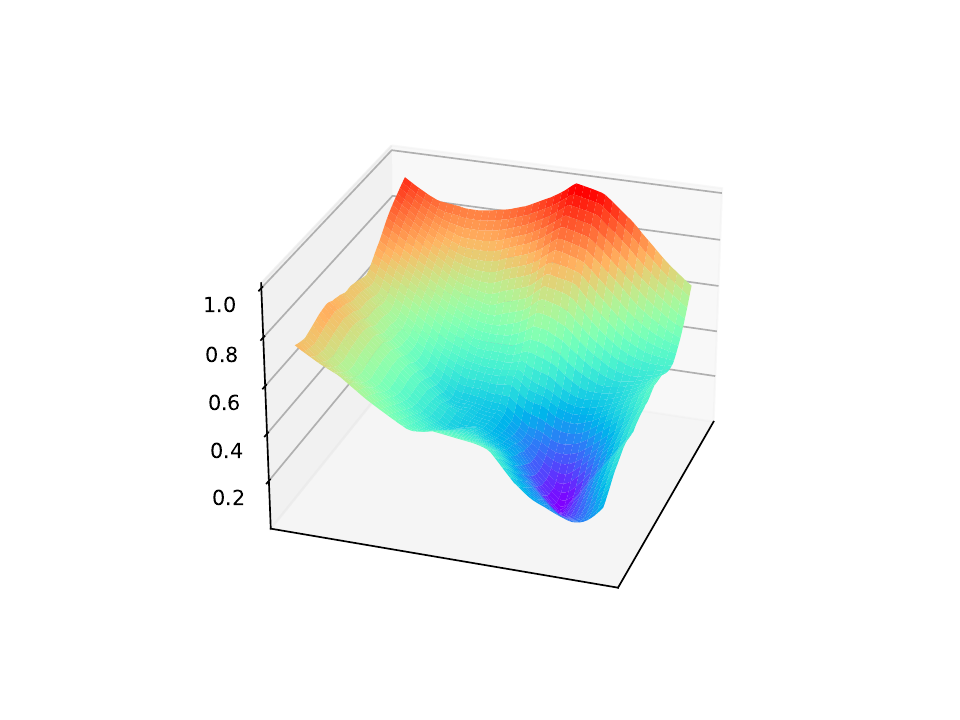} \\ 
	\end{minipage}   }  
	\subfloat[ReLU~($L=10$)]{  
		\begin{minipage}[b]{0.3\textwidth}  
			\includegraphics[width=0.48\textwidth,trim=60 50 50 50,clip]{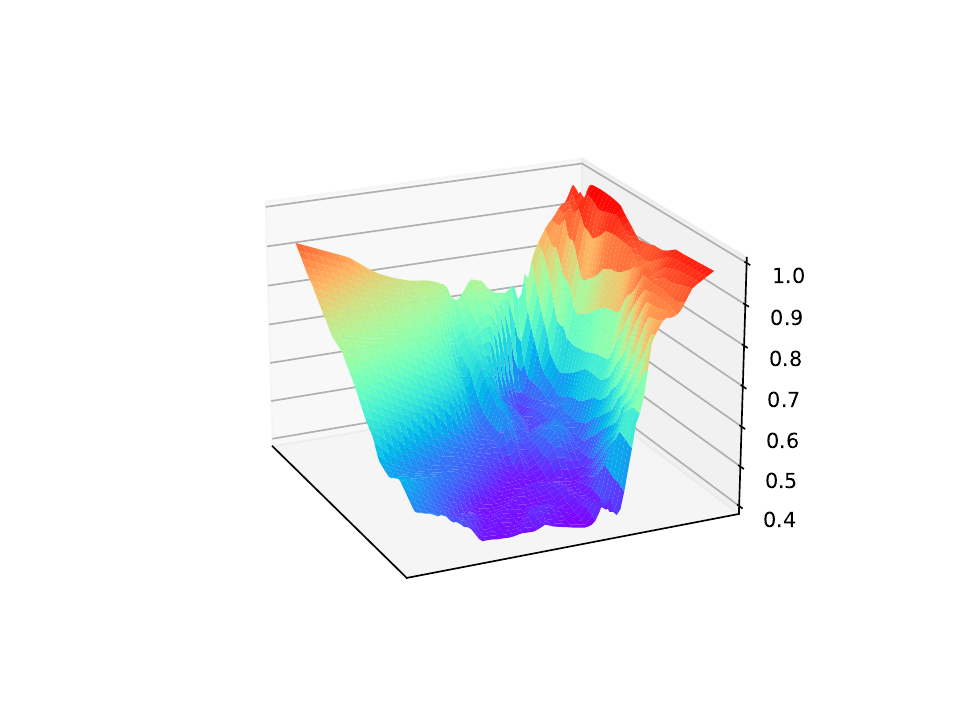}    
			\includegraphics[width=0.48\textwidth,trim=60 50 50 
			50,clip]{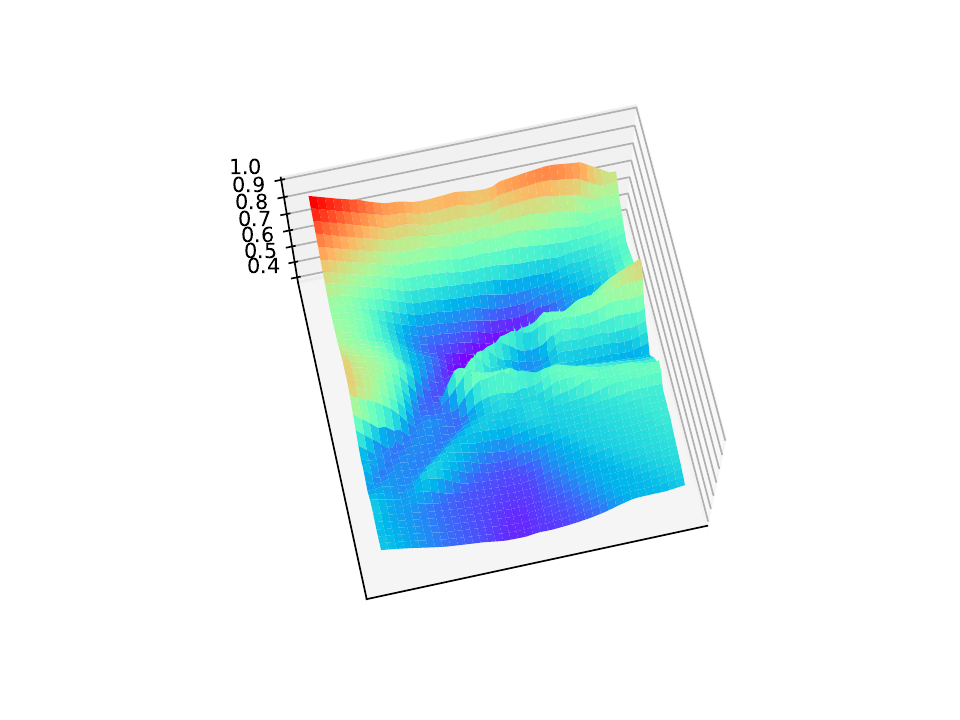} \\
			\includegraphics[width=0.48\textwidth,trim=60 50 50 50,clip]{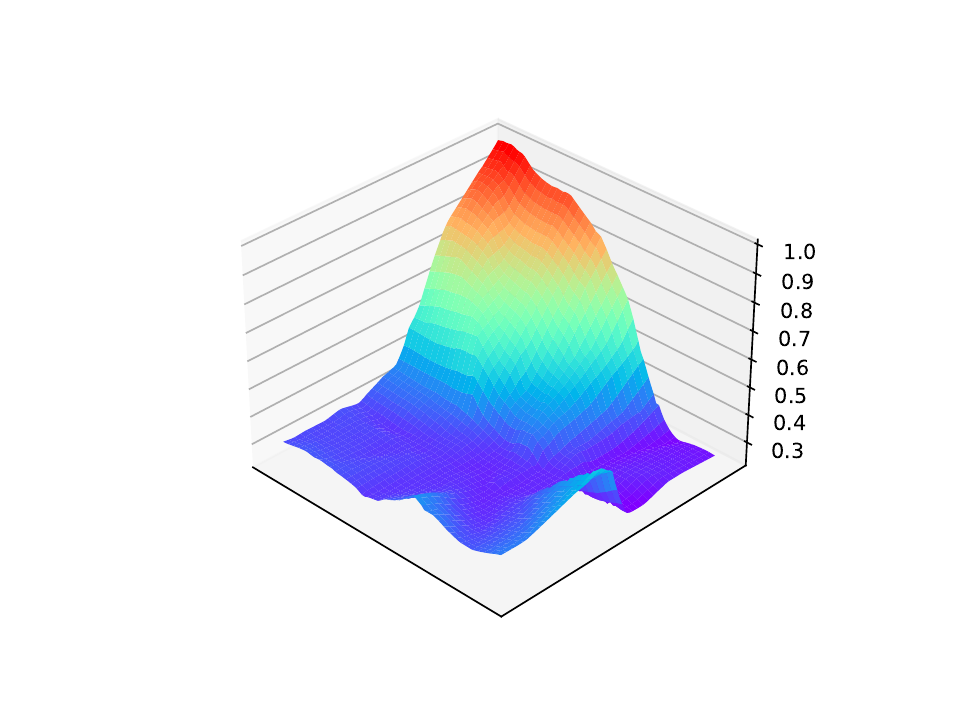}   
			\includegraphics[width=0.48\textwidth,trim=60 50 50 
			50,clip]{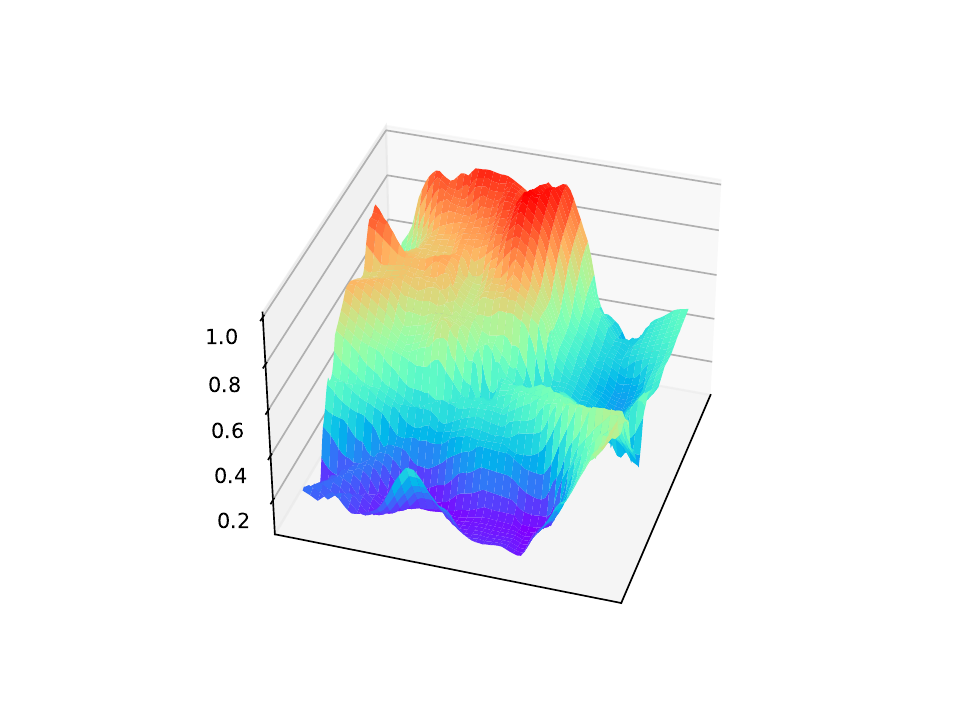} \\ 
	\end{minipage}   }   
	\subfloat[ReLU~($L=30$)]{  
		\begin{minipage}[b]{0.3\textwidth}  
			\includegraphics[width=0.48\textwidth,trim=60 50 50 
			50,clip]{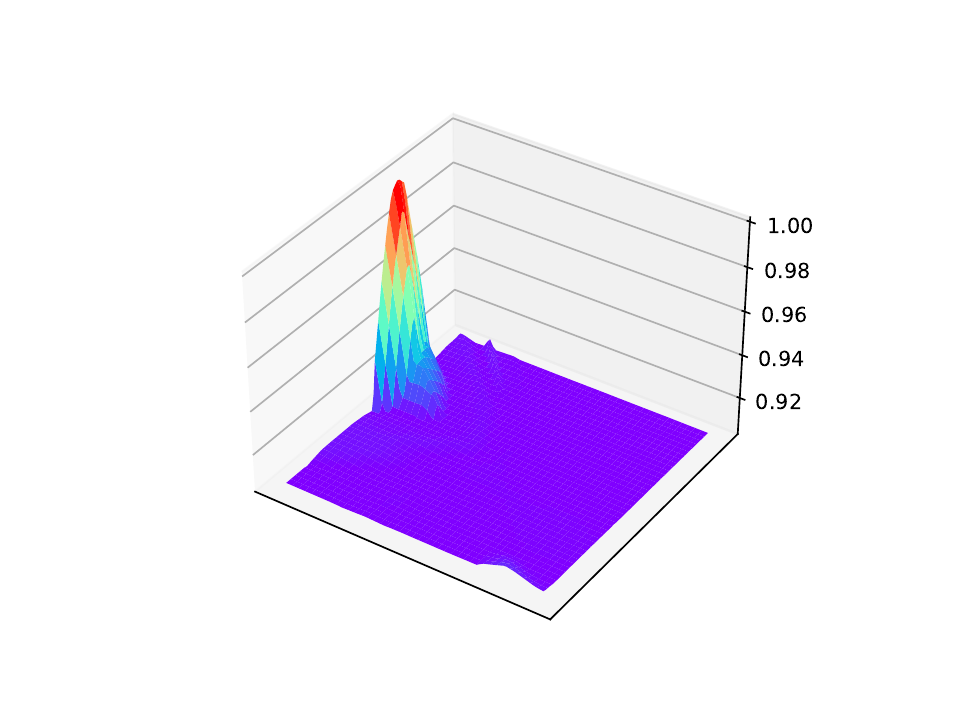}    
			\includegraphics[width=0.48\textwidth,trim=60 50 50 
			50,clip]{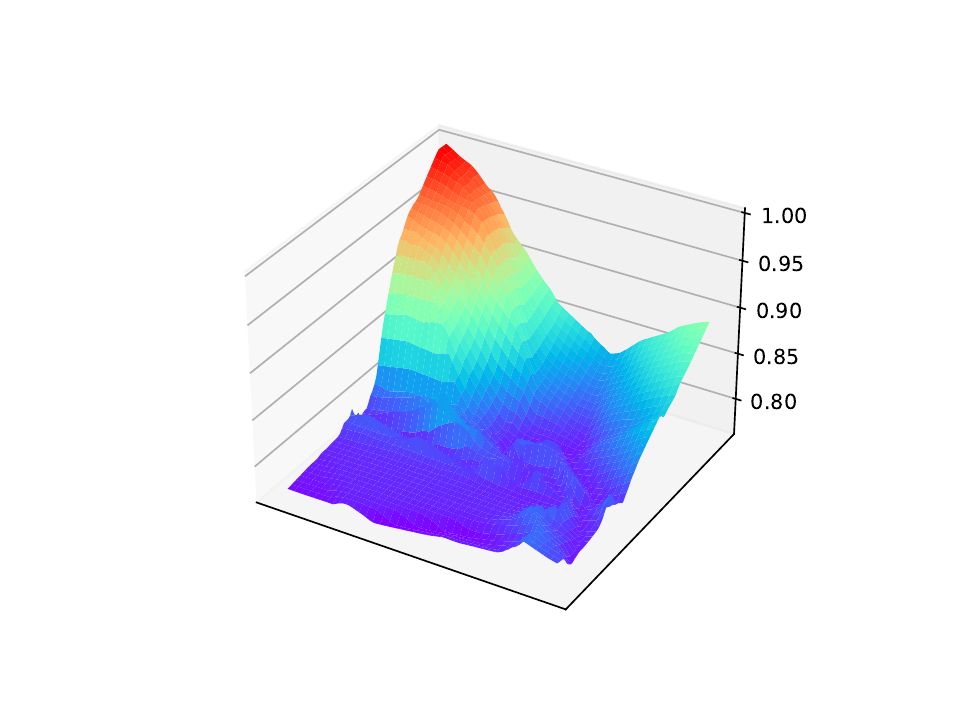} \\
			\includegraphics[width=0.48\textwidth,trim=60 50 50 
			50,clip]{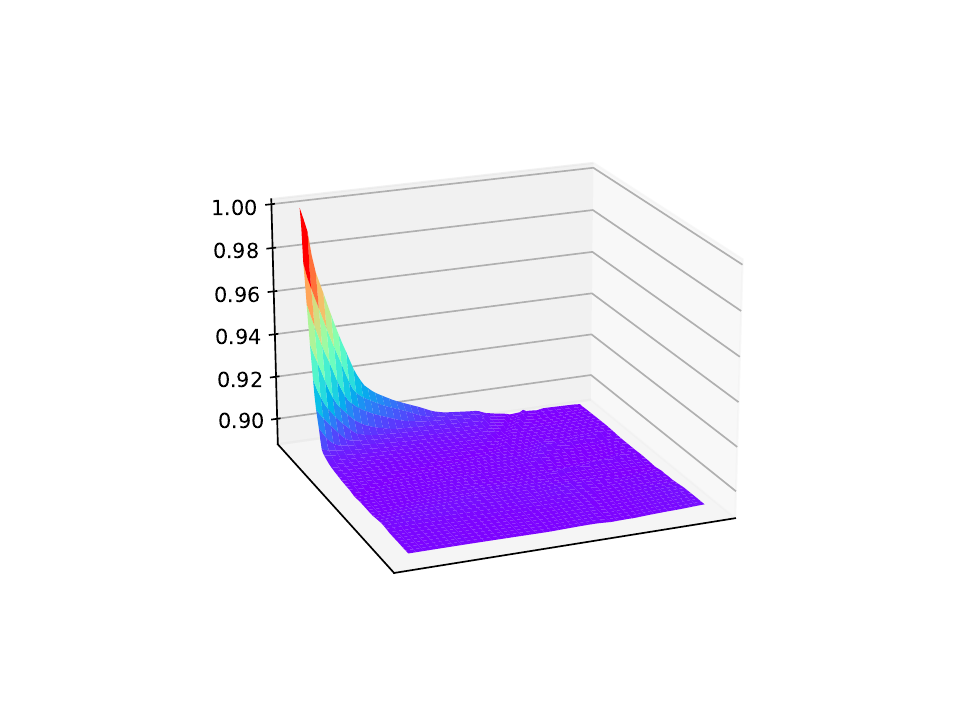}   
			\includegraphics[width=0.48\textwidth,trim=60 50 50 
			50,clip]{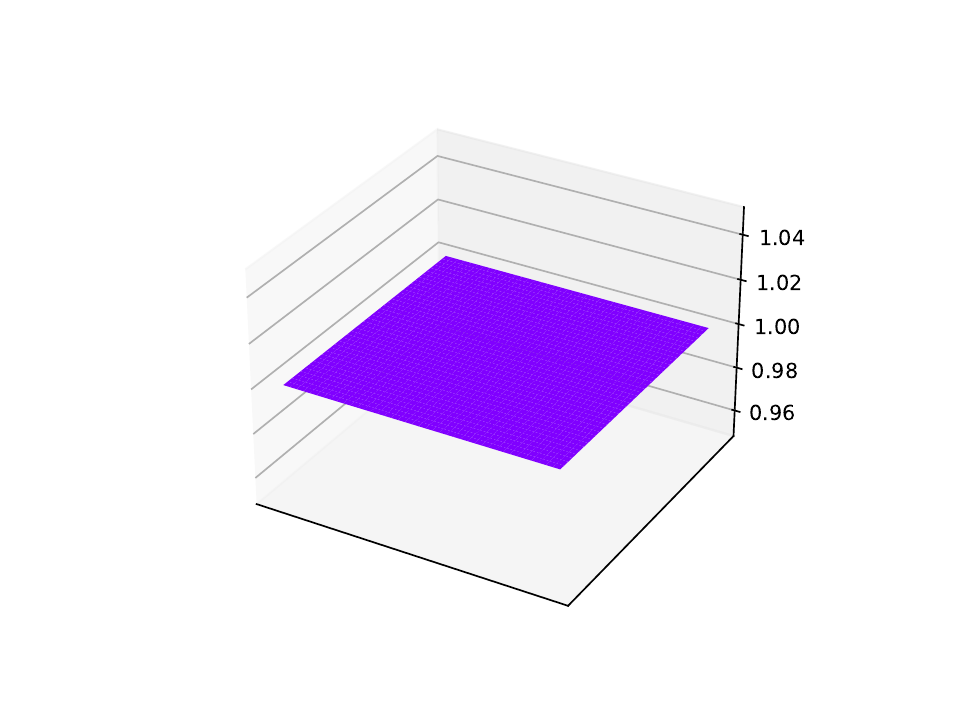} \\ 
	\end{minipage}   }   
	
	\subfloat[ABS~($L=1$)]{ 
		\begin{minipage}[b]{0.3\textwidth}  
			\includegraphics[width=0.48\textwidth,trim=60 50 50 50,clip]{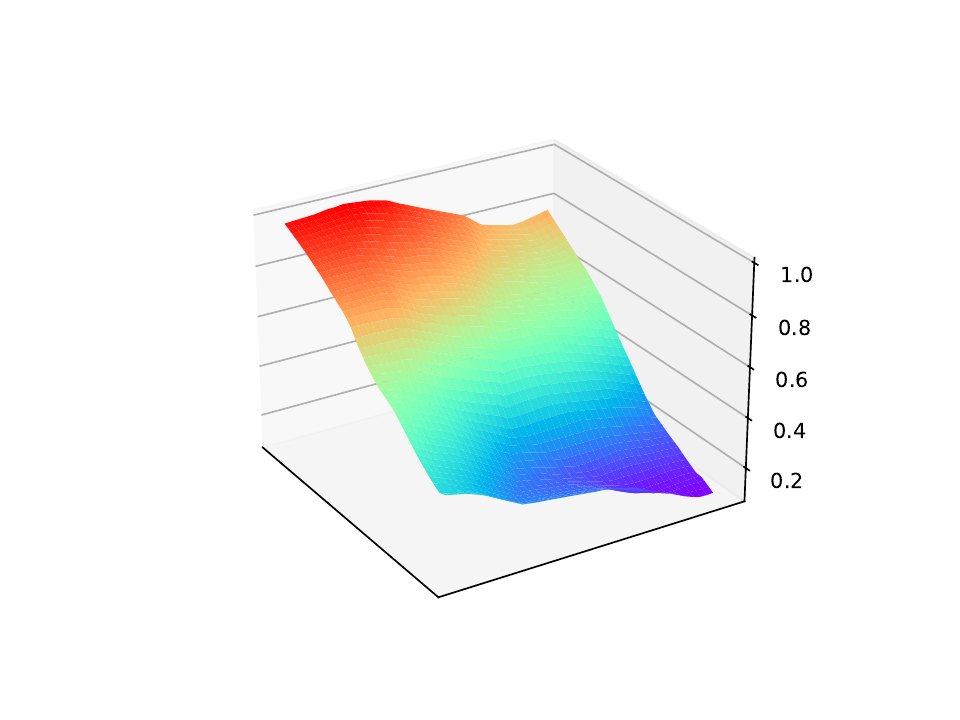}    
			\includegraphics[width=0.48\textwidth,trim=60 50 50 
			50,clip]{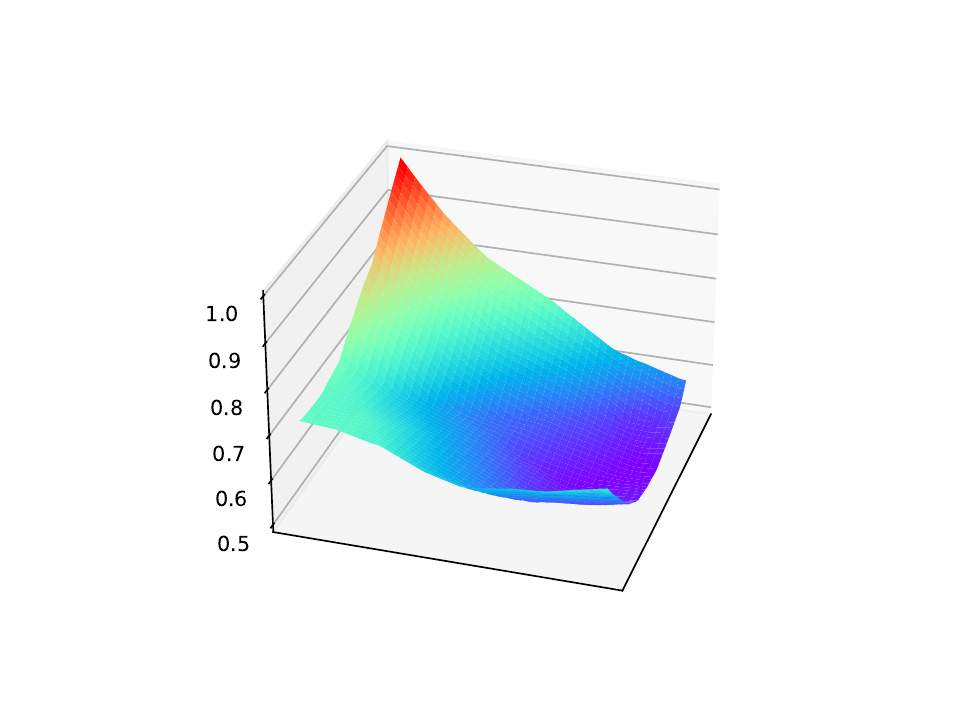} \\
			\includegraphics[width=0.48\textwidth,trim=60 50 50 50,clip]{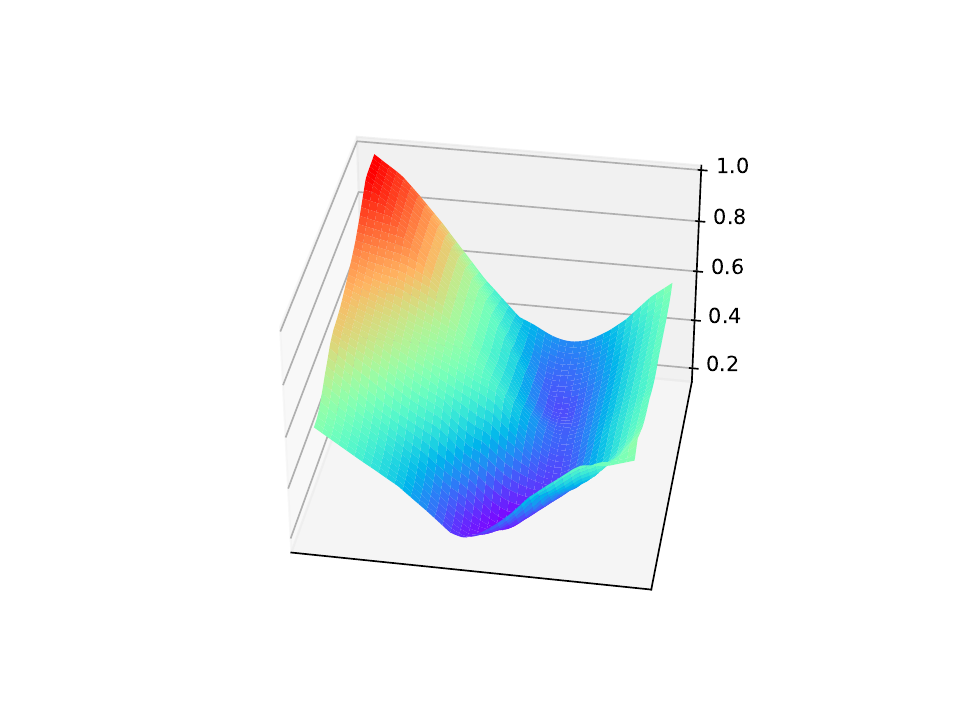}   
			\includegraphics[width=0.48\textwidth,trim=60 50 50 
			50,clip]{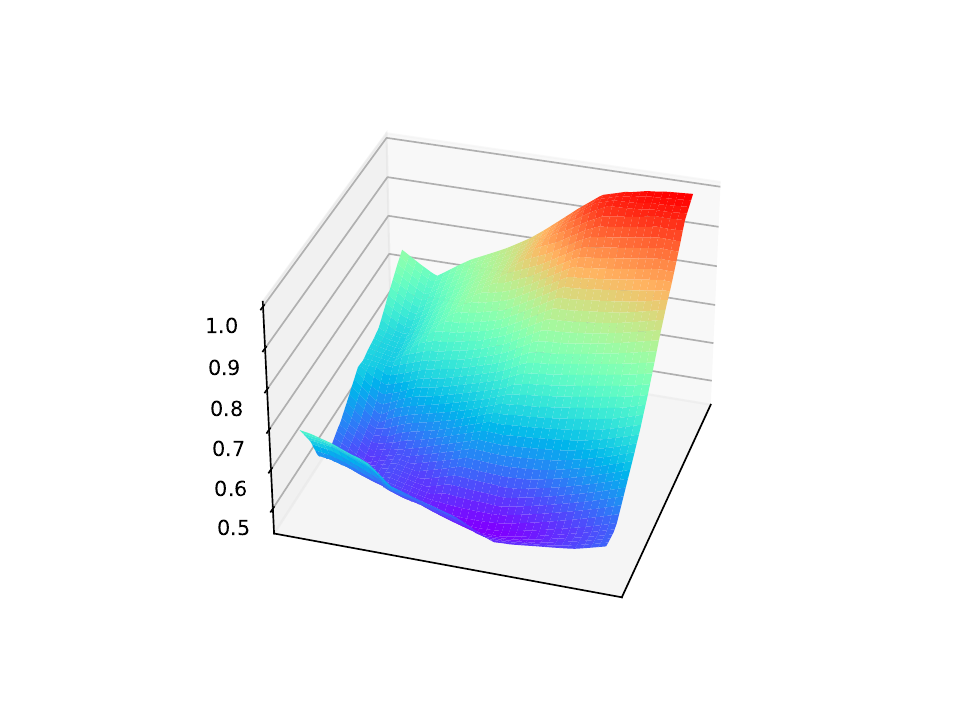} \\ 
	\end{minipage}   }  
	\subfloat[ABS~($L=10$)]{  
		\begin{minipage}[b]{0.3\textwidth}  
			\includegraphics[width=0.48\textwidth,trim=60 50 50 50,clip]{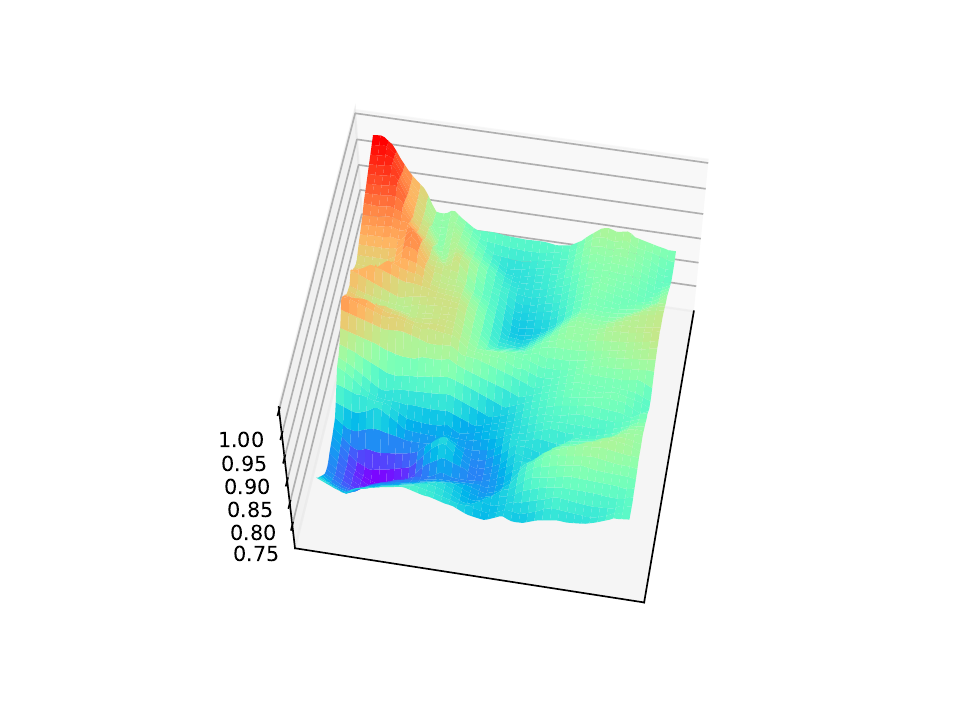}    
			\includegraphics[width=0.48\textwidth,trim=60 50 50 
			50,clip]{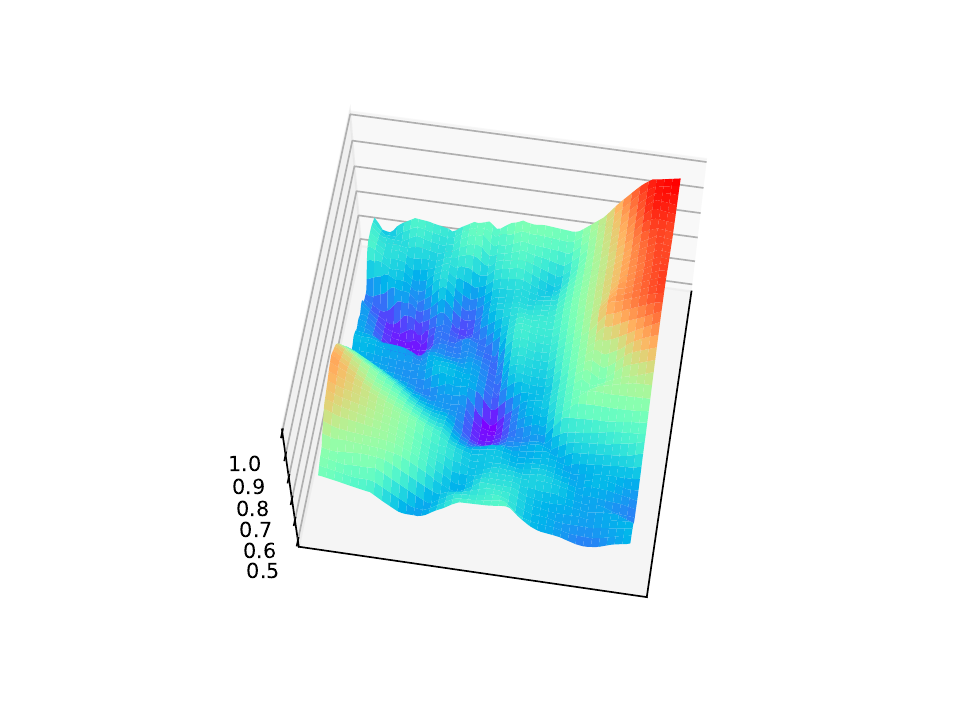} \\
			\includegraphics[width=0.48\textwidth,trim=60 50 50 50,clip]{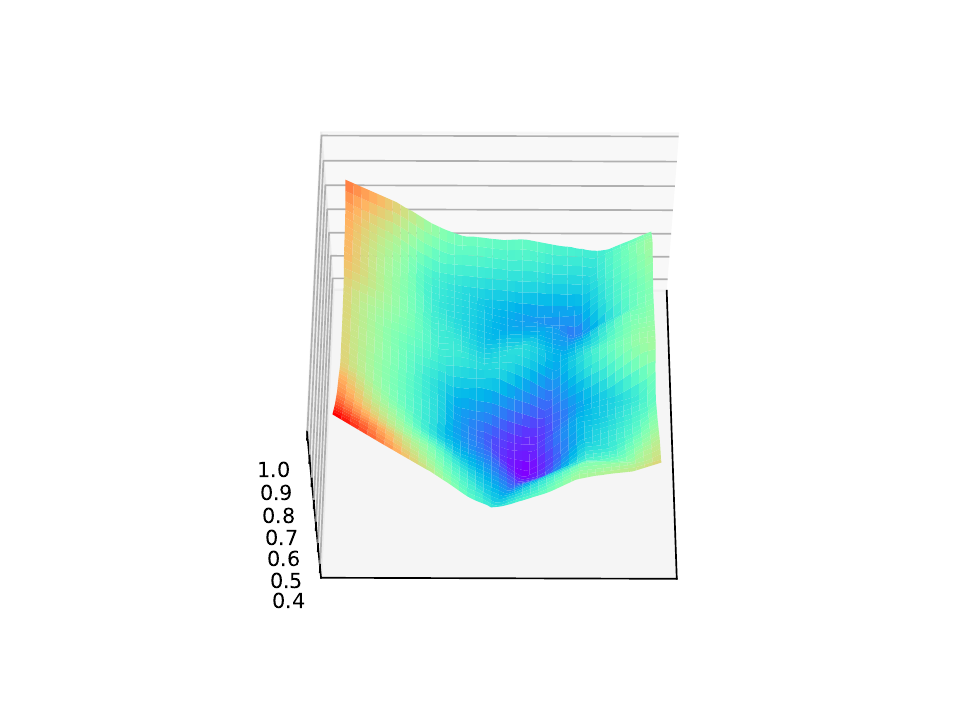}   
			\includegraphics[width=0.48\textwidth,trim=60 50 50 
			50,clip]{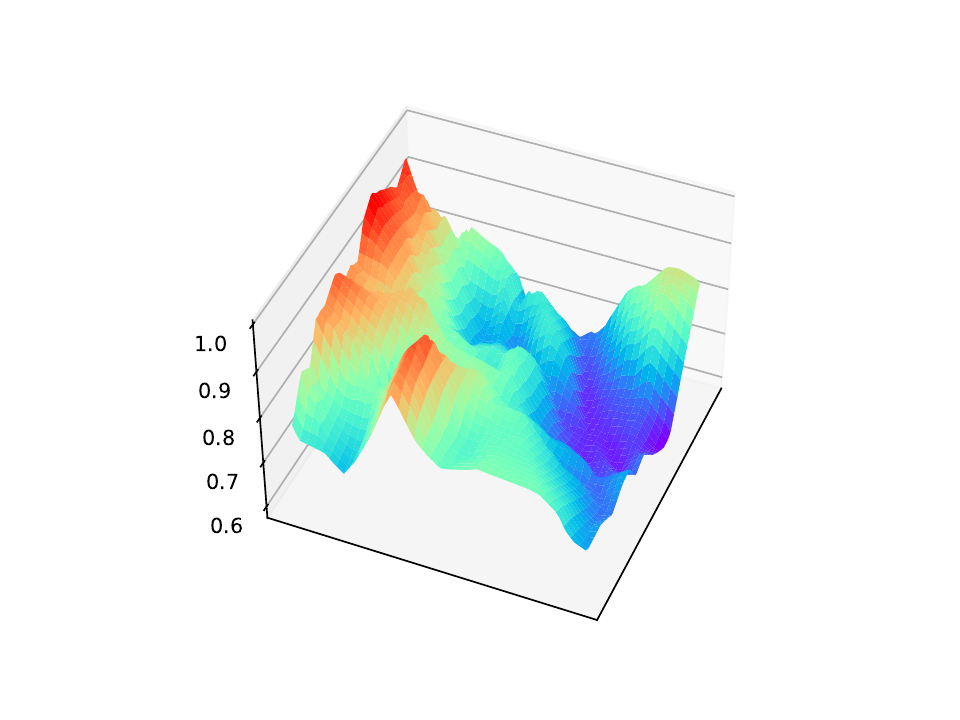} \\ 
	\end{minipage}   }   
	\subfloat[ABS~($L=30$)]{  
		\begin{minipage}[b]{0.3\textwidth}  
			\includegraphics[width=0.48\textwidth,trim=60 50 50 
			50,clip]{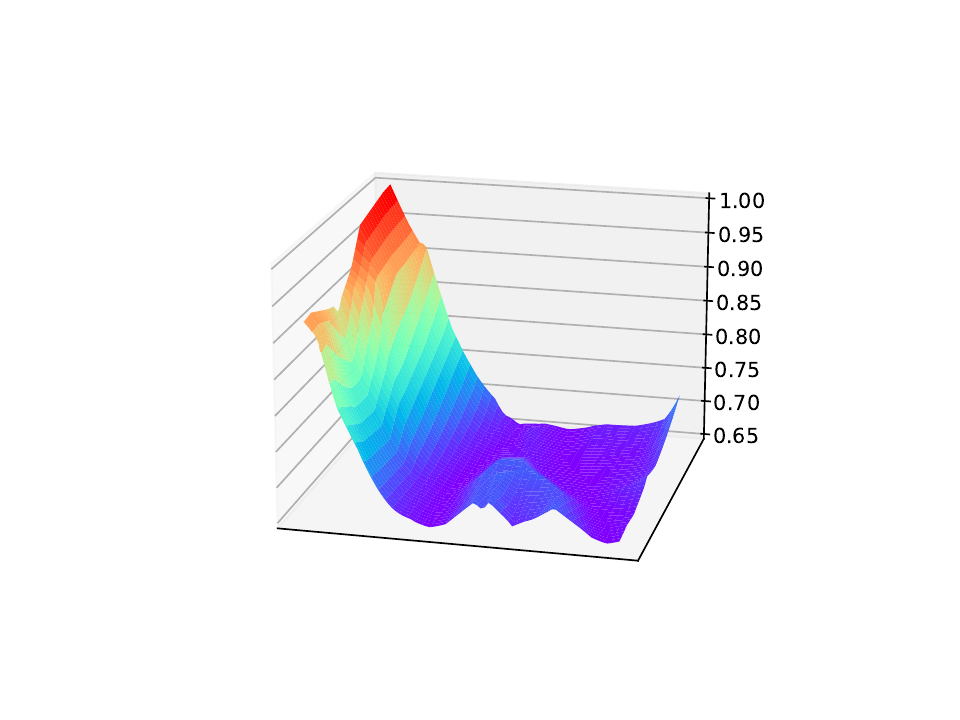}    
			\includegraphics[width=0.48\textwidth,trim=60 50 50 
			50,clip]{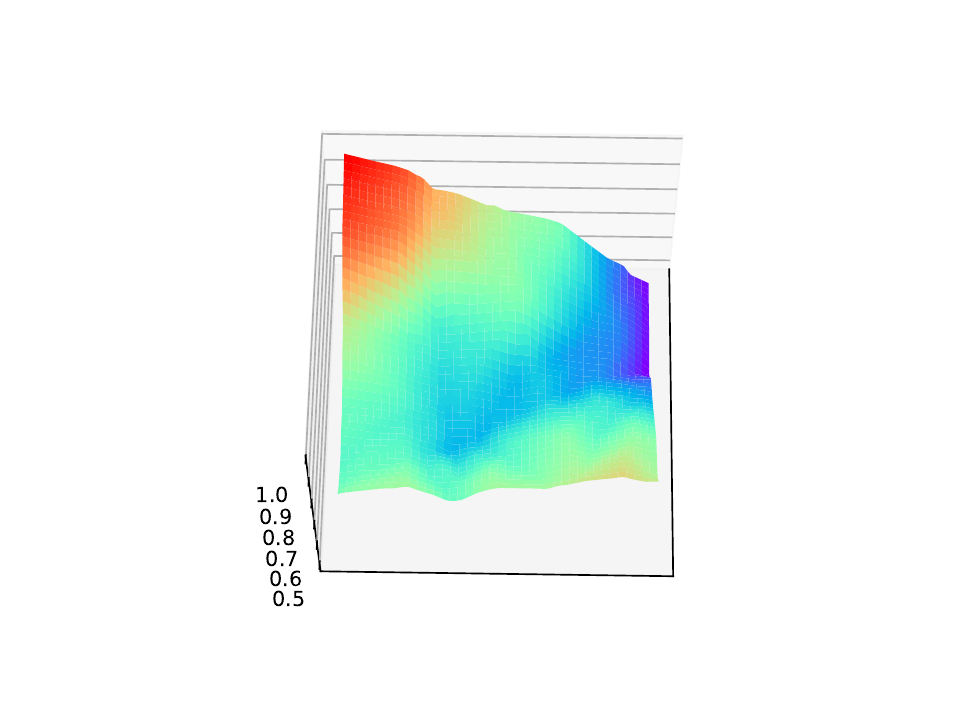} \\
			\includegraphics[width=0.48\textwidth,trim=60 50 50 
			50,clip]{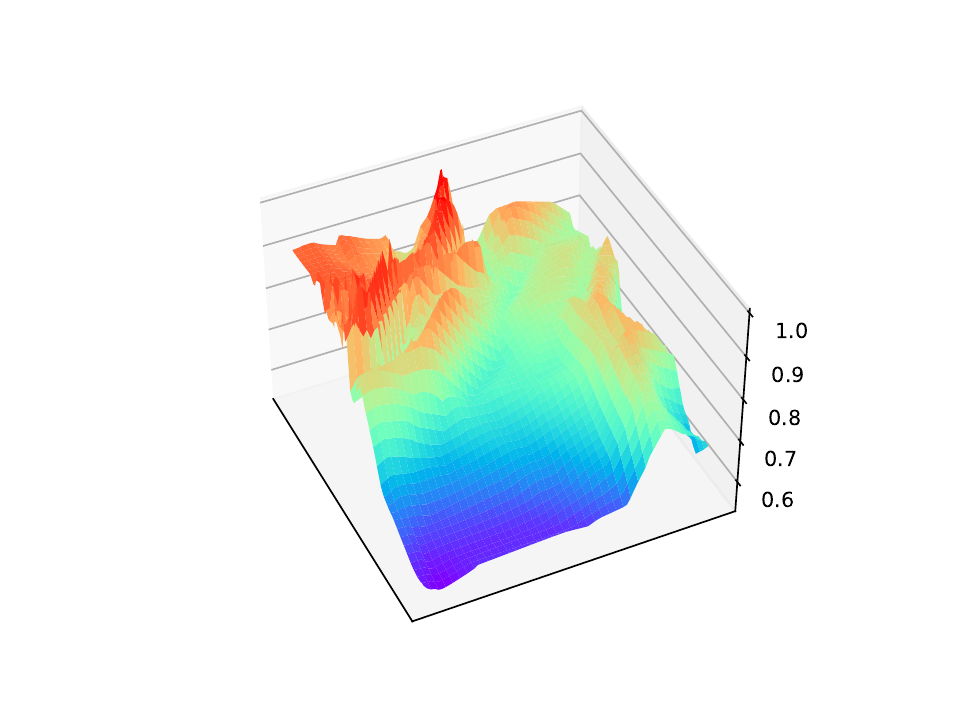}   
			\includegraphics[width=0.48\textwidth,trim=60 50 50 
			50,clip]{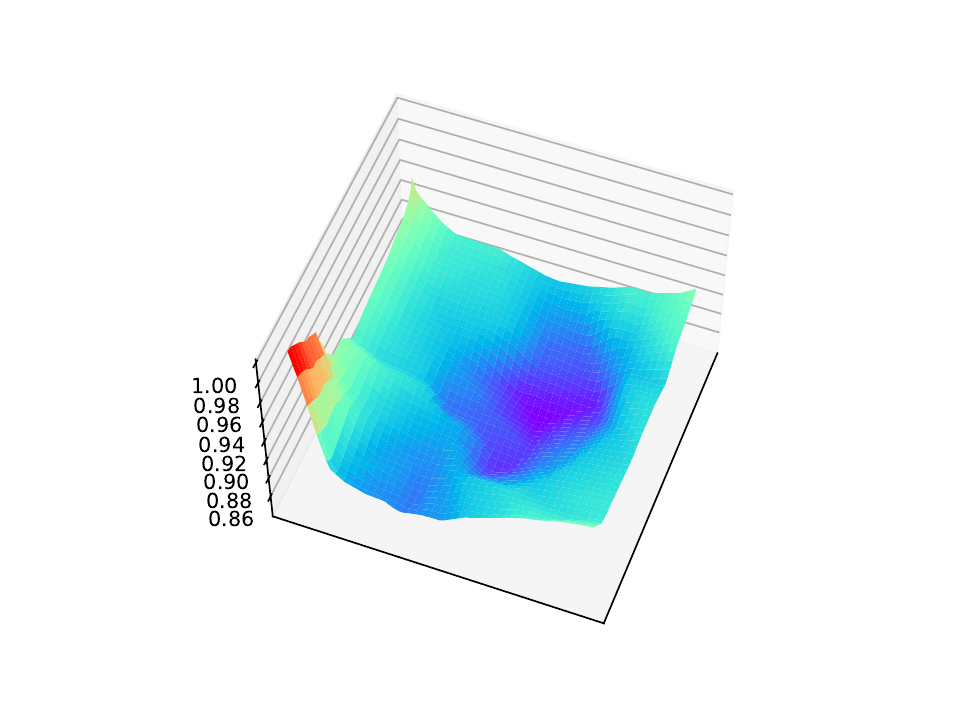} \\ 
	\end{minipage}   }   
	
	\caption{Landscape of $\|F_L(x,\bW,\bb)\|^2$. Each block consists of four landscapes obtained from different random parameter samples. }
	\label{fig:visualization-of-3network}
\end{figure*}

\emph{Variability represents the richness of landscape patterns of neural network maps in data space for well-scaled random parameters}. In particular, we consider the landscapes of the MLP function $F_L(x) := F_L(x,\bW,\bb)$ for fixed but random $\bW$ and $\bb$ that are initialized and scaled as described above.

A network function of high variability exhibits great variations in data space and is also sensitive to parameter changes. In contrast, a low variability network possesses few pattern changes in data space and is insensitive to parameter changes.  Intuitively, the latter should be more difficult to train.

\subsection{Visualizaton}
\label{sec:see-nn}

We start with a set of simple experiments in Section~\ref{sec:see-nn} to observe landscapes of MLP function $F_L(x)$ as the network depth $L$ increases while the total number of parameters $N_w$ is kept as a constant. 

We will visualize the surface $z = \|F_L(x,\bW,\bb)\|^2$ in $\R^3$ over $x \in [-1,1]^2 \subset \R^2$ for a sequence of randomly sampled, properly scaled and fixed parameter pairs $(\bW,\bb)$.  To do so, we discretize the square $[-1,1]^2$ by an $81\times 81$ uniform grid consisting of 6561 grid points.  In all cases, whenever we vary the MLP depth $L$, we adjust the width $d$ accordingly so that the total number of model parameters is fixed at $N_w=3200$.

We present our visualization results in Figure~\ref{fig:visualization-of-3network} for three activation functions (Sigmoid, ReLU and ABS) and three depth values.  For each of these nine cases, we present a block of four plots for $z=\|F_L(x,\bW,\bb)\|^2$ (where $z$ has been scaled into the range $[0,1]$) corresponding to four different random values of $(\bW,\bb)$.

\subsection{Observations}
From Figure~\ref{fig:visualization-of-3network}, we make the following observations.
\begin{enumerate}
\item 
In the first row for the Sigmoid function, the surfaces are rather monotonous with few up-and-down variations either in the data space (within each plot) or in the parameter space (across plots).
Most remarkably, at a depth of $L=20$, the surfaces essentially become constant.  For brevity, we will call this phenomenon as  "Collapse to Constant" or simply C2C.
\item 
In the second row for the ReLU function, we find landscapes with much richer expressions in the data space, and with an increasing amount of variations as $L$ grows from 1 to 10. However, C2C also occurs for the ReLU function at $L=30$. (Further experiments, not presented here, confirm that other ReLU-like functions, such as Leaky-ReLU, suffer from C2C as well).
\item
In the third row for the ABS function, the plots follow the similar trend as for the ReLU function for $L=1$ to 10.  On the other hand, C2C has not yet occurred when $L=30$, although the variations in the landscape appear to have started diminishing.
\end{enumerate}
Upon further examinations, it is clear that for large $L$ not only the scalar function $\|F_L(x,\bW,\bb)\|$ tends to constants in $x$-space, but in fact the vector-valued function $F_L(x,\bW,\bb)$ tends to constant vectors for $x \in [-1,1]^2$.  We will explain this C2C phenomenon later.

From the experimental results presented in this section, we witness unmistakable differences in the outputs of MLP functions $F_L(x,\bW,\bb)$ as the result of different activation functions and different values of the depth $L$.

\section{Two measures of variability}

In general, measuring variability in neural networks is a critical yet highly challenging task. This paper proposes two sensible measures, which, however, are not yet suitable for high-dimensional data. Despite their limitations, these metrics provide new insights into neural networks.

\subsection{Variability arising with activation ratio }

It should be clear that the nonlinear activations in the model are the source of variability.  As we see from the previous variability visualizations, variability initially always arises.  There is a simple explanation for this.  That is, when the total number of parameters is fixed,  the number of activations always increases with the depth.  

To further illustrate this point, let us consider MLPs \eqref{def:Fk} with the number of parameters \eqref{eq:Nw}. In this case, the total number of activations is $(L+1)d$. The \emph{activation ratio} $\rho$ of this network is defined as the total number of activations divided by the total number of parameters, i.e., $(L+1)d/N_w$, which represents the average number of activations per weight, $L\ge1$,

\begin{small}
\begin{equation} \label{eq:activation ratio}  
	\rho(L)~ = \frac{(L+1)\left(\sqrt{(L+5)^2+4(N_w-2)L}-L-5\right)}{2N_wL},
\end{equation}
\end{small}
the activation ratio is monotonically increasing with $L$.

It is worth noting that the activation rate is a simplistic and crude metric, as it does not take into account the type of activation functions, which can clearly make huge differences in landscape patterns as our visualization experiment indicates. Moreover, the ratio $\rho(L)$ is monotonically increasing.  As we will show soon,  it would eventually deviate from the ``true variability" as $L$ become ``too large", even though it works quite well (for activation functions RELU and ABS) before $L$ becomes too large.

\begin{figure*}[htb]
	\begin{center}
		\vspace{-.1cm}
		\subfloat[$N_w=1600$, Sigmoid~(left), ReLU~(middle) and ABS~(right).]{
			\includegraphics[width=.32\textwidth,trim=0 20 20 0, clip]{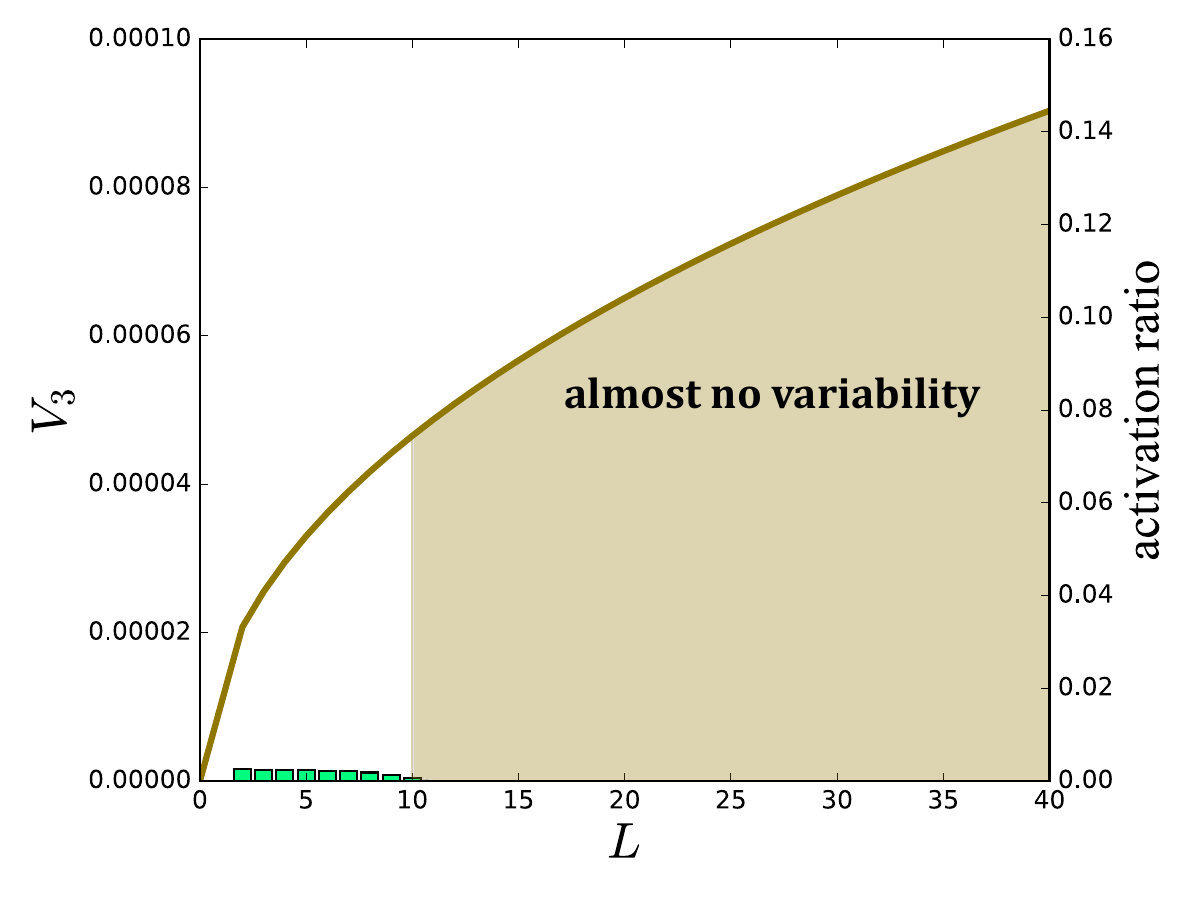}
			\includegraphics[width=.32\textwidth,trim=0 20 20 0, clip]{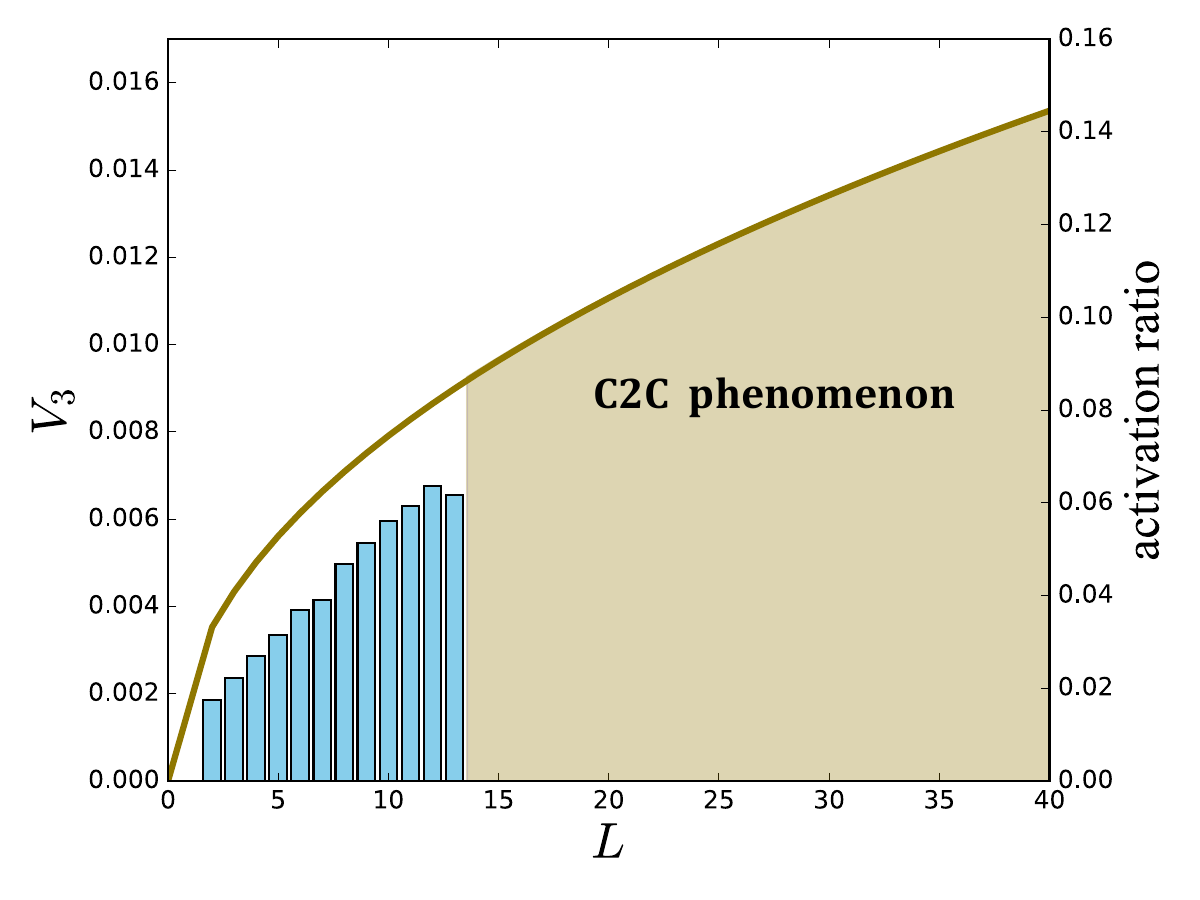}
			\includegraphics[width=.32\textwidth,trim=0 20 20 0, clip]{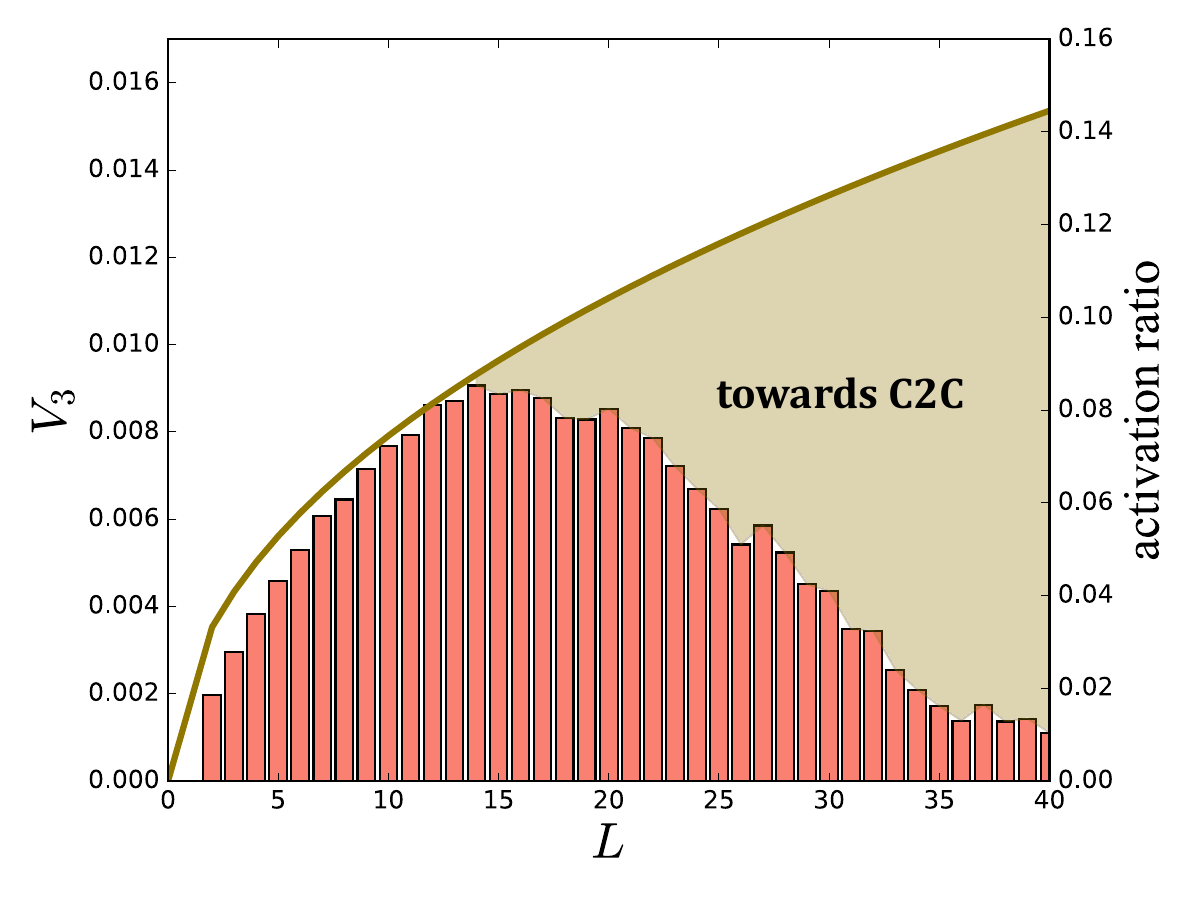}}~~
		
		\subfloat[$N_w=3200$.]{
			\includegraphics[width=.32\textwidth,trim=0 20 20 0, clip]{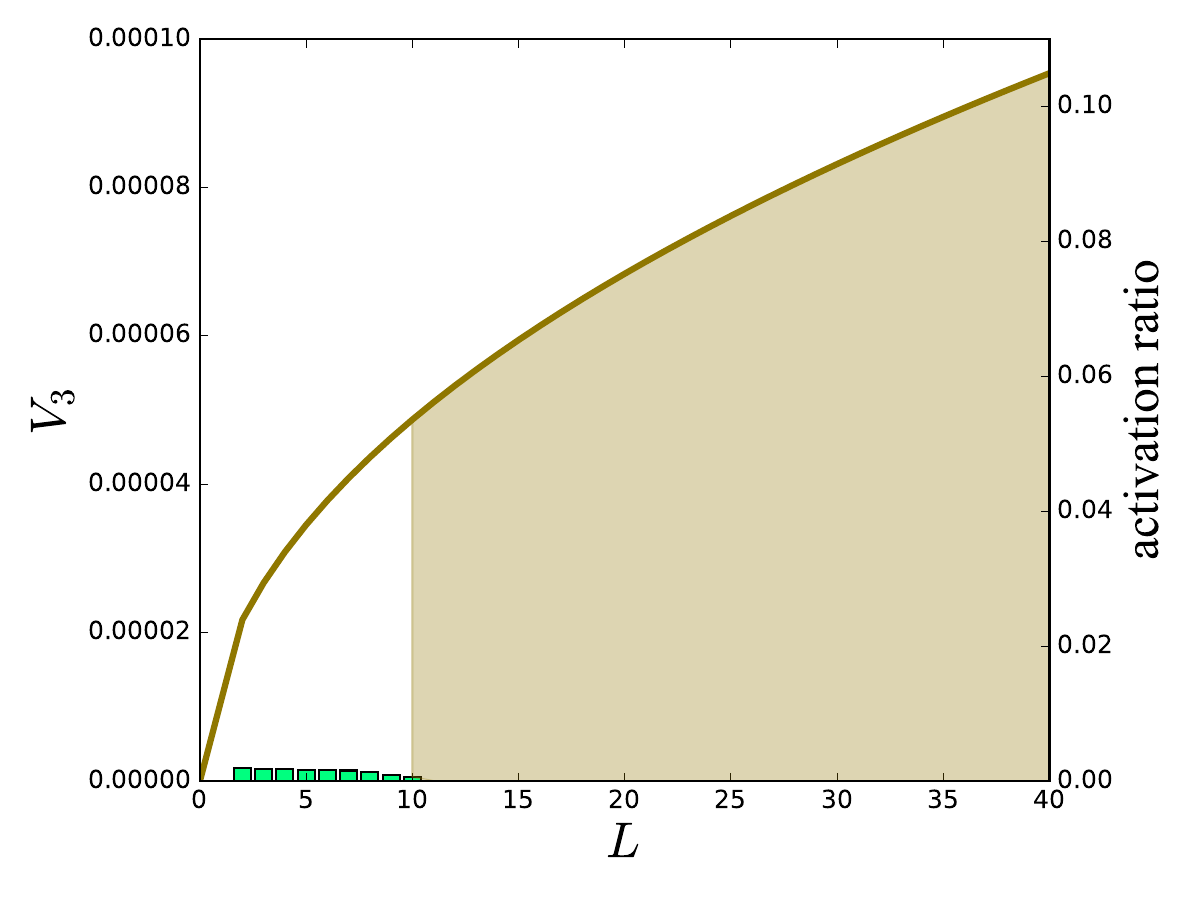}
			\includegraphics[width=.32\textwidth,trim=0 20 20 0, clip]{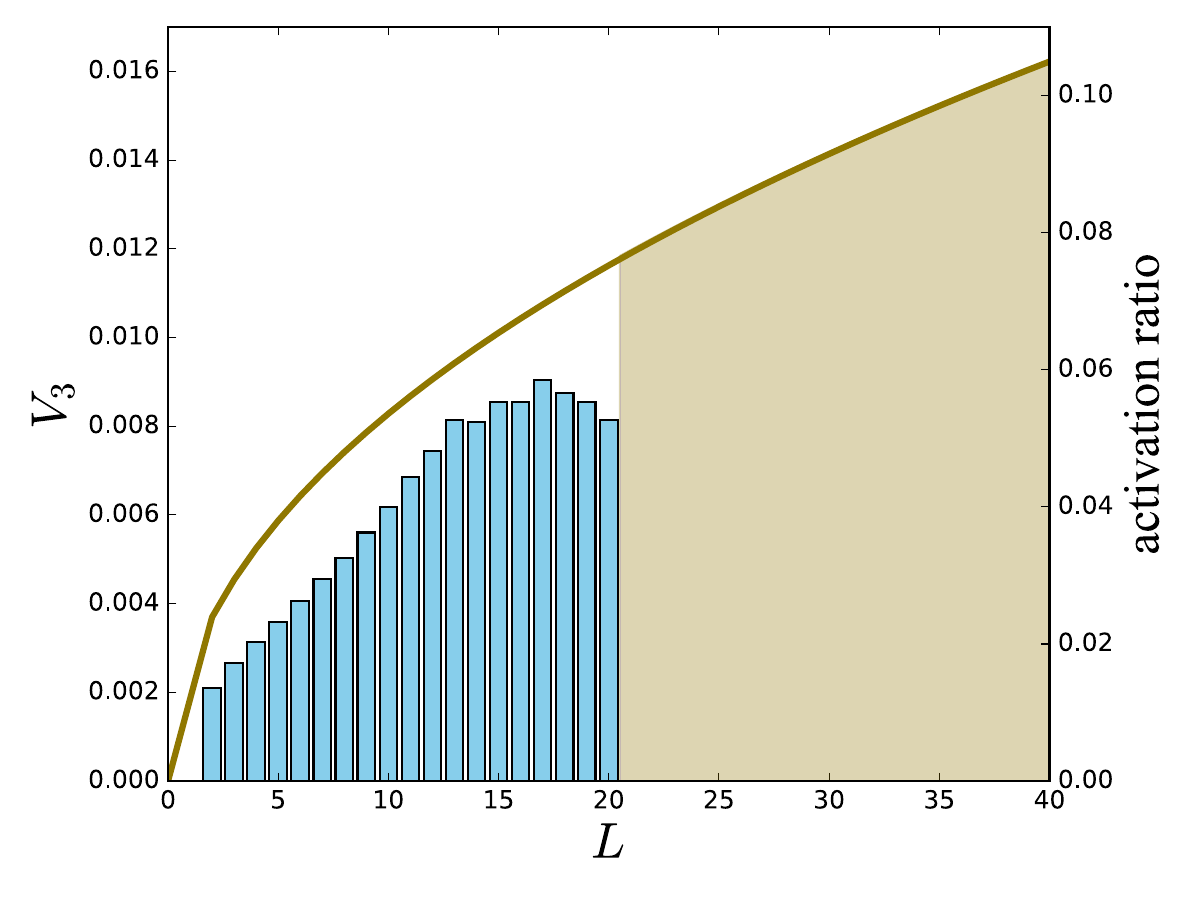}
			\includegraphics[width=.32\textwidth,trim=0 20 20 0, clip]{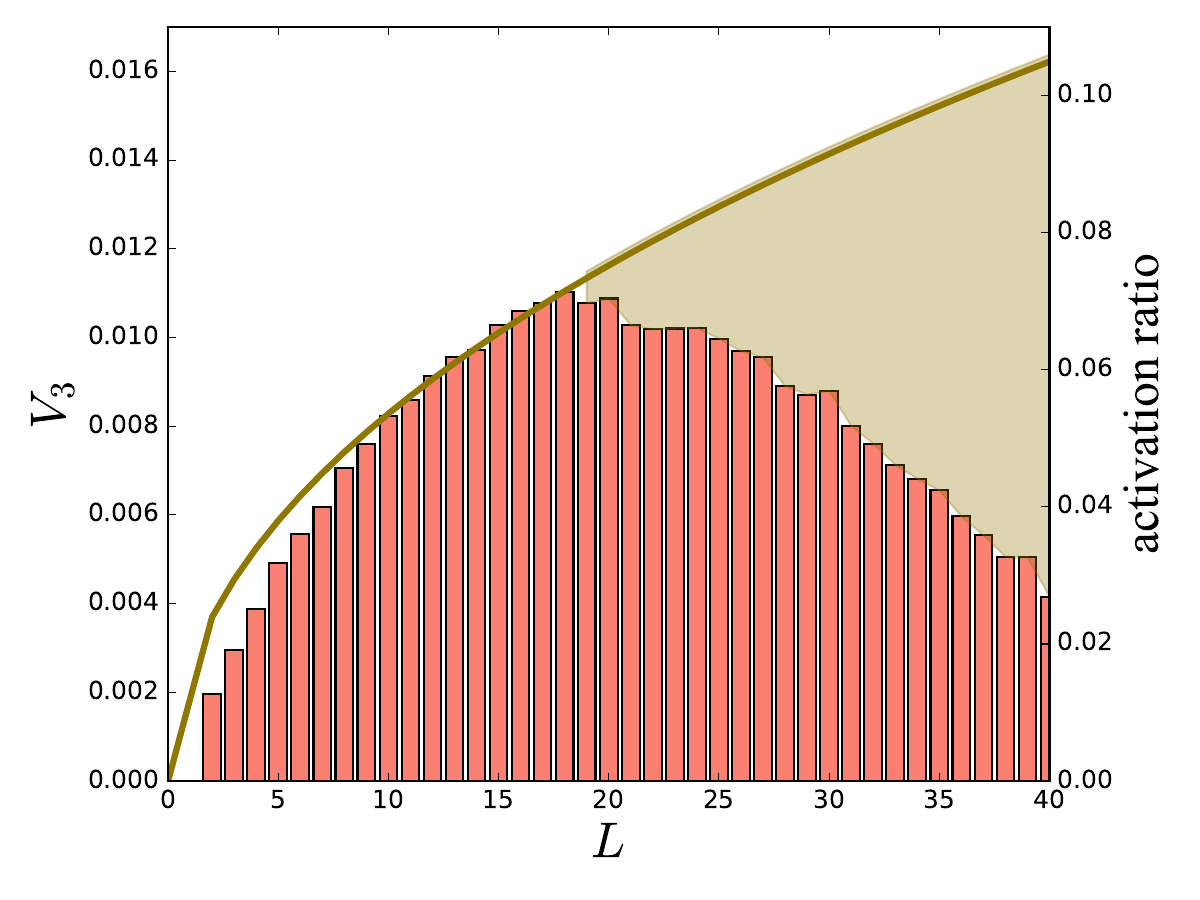}}~~
		
		\vspace{-.1cm}
		\caption{Variability measured $V3$ for three activation functions. Each bar represents a geometric mean of 3000 parameter samples. As depth $L$ grows, the width $d$ decreases so that the total number of model parameters is approximately 1600 (in the first row) or 3200 (in the second row). }
		\label{fig:activation-var}
	\end{center}
	\vspace{-.5cm}
\end{figure*}

\subsection{A more accurate measure:  $V_3$}

We present another variability measurement that has worked reasonably well in the scope of the current work. 
Under the assumption of differentiability, we define
\begin{equation}\label{def:V3}
	\mathrm{V_3} :=
	\mathbb{E}_{(\bW,\bb)}\!\!\left[\|f\|^{-1}_\infty~
	\mathbb{E}_{x\in \Omega}\left(\sum_{i=1}^d\left|\frac{\partial^3 f}{\partial^3 x_i}\right|\right)\right]
\end{equation}
where $f:= \|F_L(x;\bW,\bb)\|_2^2$, $\Omega$ is the data domain (i.e., $[-1,1]^2$ in our case), and $(\bW,\bb)$ are random parameters specified by the initialization schemes in \ref{scale}.  

The quantity $V_3$ measures the relative size of the third partial derivatives of $f$ with respect to each variable over the data domain, and then takes a mean value over the relevant random parameters.

We provide a brief justification for formula \eqref{def:V3}, in particular the use of the third derivatives.  For any fixed random parameters, if $F_L(x;\bW,\bb)$ is linear in $x$, then $f$ is quadratic and $V_3 = 0$, meaning linear networks have no variability.  This would not be the case if variability were defined only through the second derivatives. Conversely, models with high variability should have relatively large third derivatives with high probability.  For computational efficiency, we only use the principal third-order derivatives while ignoring all the cross-derivatives.

Practical MLP models, such as those activated by RELU or ABS, are often non-differentiable. For this reason and also for computational reasons, in our experiments we replace the third derivatives in \eqref{def:V3} by third-order finite differences instead, that is, on a uniform, $81 \times 81$ grid over $[-1,1]^2$, we compute
\[
\frac{\Delta^3 f}{\Delta x_i^3}\approx\frac{\partial^3 f}{\partial^3 x_i},
\]
and take the empirical algebraic means over the afore-mentioned grid on $[-1,1]^2$.

Additionally, in computing \eqref{def:V3} we use empirical geometric means by randomly sampling the parameters $(\bW,\bb)$.  Clearly, we will get $V_3=0$ if a single sampled value is zero inside the brackets on the right-hand side of \eqref{def:V3}.  Consequently, variability measure $V_3$ vanishes once C2C occurs; i.e., the network output $F_L(x;\bW,\bb)$ becomes a constant on the grid.

Figure~\ref{fig:activation-var} shows how $V_3$ varies as MLP depth $L$ increases from 1 to 40 while the parameter number $N_w$ is fixed to either 1600 or 3200.  The empirical geometric mean values are taken over 3000 parameter samples. The tested MLPs are activated by one of the three functions: Sigmoid, ReLU and ABS.  The activation ratio $\rho(L)$, which is a quantity independent of activation functions, is plotted along with computed $V_3$ values.  

As we can see from Figure~\ref{fig:activation-var}, Sigmoid hardly has any variability, as is measured by $V_3$, in comparison to both ReLU and ABS.  For the latter two, $V_3$ follows quite closely with the trend of $\rho(L)$ until the depth $L$ grows  too deep.  Then, C2C phenomenon occurs for ReLU promptly after $V_3$ passes its peak value.  On the other hand, for ABS the C2C phenomenon progresses much more gradually and slowly.  This observation suggests that, under the standard initialization scheme, ABS activation can provide MLPs with much higher variability than ReLU does.

We note that, as it is defined in \eqref{def:V3}, the variability measure $V_3$ cannot be efficiently applied for high-dimensional data due to a curse of dimensionality.  On the other hand, the formula might still be useful for sampling low-dimensional subspaces to obtain partial information on model variability.

\section{Collapse to Constant and its characterization}
In this section, $F_L(\cdot)$ is the all hidden layer function
\begin{equation}
	F_L(\cdot)= (\psi_L\circ\cdots\circ \psi_1)(\cdot): \R^d \rightarrow\R^d.
\end{equation}

To compute the derivative of $F_L(x,\bW,\bb)$ with respect to the parameters, one uses the chain-rule to obtain so-called back-propagation formulas, such as
\begin{equation*}
\left[ \frac{\partial}{\partial b_1} F_L(x,\bW,\bb) \right]^{\T} = 
D_{\phi}(z_1)W_2^{\T}D_{\phi}(z_2) \cdots  W_{L}^{\T} D_{\phi}(z_L) ,
\end{equation*}
where $z_k$, for $k=1, \cdots, L$, are computed in \eqref{def:recursion} and $D_{\phi}(z_k)$ are diagonal matrices with scalar-valued $\phi'$ applied component-wise to $z_k$.  It is well-known that the behavior of the derivatives is critically determined by the properties of the above matrix product. For convenience and without loss of generality, we add $W_1$ to the product and define
\begin{equation}\label{def:G_L}
G_L \equiv G_L(x,\bW,\bb) := \prod_{k=1}^L W_k^{\T} D_{\phi}(z_k),
\end{equation}
which we will simply call the $G$-matrix at $x$ associated with the network $F_L(x,\bW,\bb)$.  It is well known in deep learning that excessively small (or large) size of $G_L$ causes vanishing (or exploding) gradient, which is a major source of difficulty in training.

Now we define another matrix product called the $C$-matrix, by replacing the derivative $D_{\phi}(z_k)$ in \eqref{def:G_L} by the finite difference between two points $z_k$ and $\bz_k$; that is,
\begin{equation}\label{def:C_L}
C_L \equiv C_L(x,\bx,\bW,\bb) := \prod_{k=1}^L W_k^{\T}\hat{D}_\phi(z_k,\bz_k),
\end{equation}
where $\hat{D}_\phi(\cdot,\cdot) \in \R^{d\times d}$ is a diagonal matrix defined by 
\begin{equation}\label{def:diagD}
\left[\hat{D}_\phi(u,v)\right]_{ii} = \frac{[\phi(u)-\phi(v)]_i}{[u-v]_i}, \;\; i = 1,\cdots,d,
\end{equation}
with the convention $0/0 = 1$, $\{z_k\}$ and $\{\bz_k\}$ are computed via the recursion \eqref{def:recursion} starting from $x$ and $\bx$, respectively.  By their definitions, it is clear that, for differentiable activation functions, G-matrices are limits of C-matrices.   Nevertheless, in general these two types of matrices are different.

The next proposition shows that C-matrices characterize the C2C phenomenon.
\begin{proposition}
Let network $F_L(x,\bW,\bb): \R^{d} \rightarrow \R^{d}$ be defined as in \eqref{def:Fk}.  
For any two distinct points $x,\bx \in \R^{d}$, there holds
\begin{equation}\label{eq:dF}
F_L(x)-F_L(\bx) = \left[C_L(x,\bx)\right]^{\T}(x-\bx).
\end{equation}
Consequently, $\lim_{L\rightarrow\infty} C_L(x,\bx) = 0$ implies
 \begin{equation}\label{diff2const}
\lim_{L\rightarrow\infty}(F_L(x)-F_L(\bx)) = 0. 
\end{equation}
\end{proposition}
The verification of this proposition is straightforward so we omit it.

We note that the difference going to zero in \eqref{diff2const} does not imply that each individual sequence goes to the same limit.  On the contrary,  limits generally do not exist if the bias sequence $\{b_k\}$ is bounded away from zero.  

Regarding the C2C phenomenon, the following remarks are in order.
\begin{itemize}
\item 
Wherever $C_L(x,\bx)$ is sufficiently small in some induced matrix norm $\|\cdot\|$, the output values of the network for the two inputs $x$ and $\bx$ will be close to each other.
\item
$C_L(x,\bx)$ will be small if $\|W_k^{\T}\hat{D}_\phi(z_k,\bz_k)\|$ are sufficiently smaller than 1 for sufficiently many $k \in \{1,\cdots,L\}$.
%which occurs when one or both of the matrices is sufficiently small.
\item
If $C_L(x,\bx)$ is sufficiently small for all $(x,\bx)$ in some region, then the corresponding outputs of the network will be like a constant in that region.  In particular, this can happen when $W_k$ are small for all or many $k$.
\end{itemize}

If the weight matrices $W_k$, $k=1,\cdots,L$, are properly normalized (for example, all $W_k$ are orthogonal matrices), then the size of $C_L$ will be determined by that of $\hat{D}_\phi(z_k,\bz_k)$ for a given point pair $(x,\bx)$, which in turn depends on activation $\phi$ in use. 
We now consider ReLU and ABS functions. In both functions, the diagonal entries defined in \eqref{def:diagD} all lie in the interval $[-1,1]$.
%satisfy $[\hat{D}_\phi(u,v)]_{ii} \in [-1,1]$.
%If the weight matrices $W_k$, $k=1,\cdots,L$, are properly normalized (for example, all $W_k$ are orthogonal matrices), then the size of $C_L$ will be determined by that of $\hat{D}_\phi(z_k,\bz_k)$ for a given point pair $(x,\bx)$, which in turn depends on activation $\phi$ in use. 

\begin{proposition}
Suppose that $u,v\in\R^d$ be i.i.d. random variables with 
\[
\pr([u]_i\ge 0) =  \pr([v]_i\ge 0) = p \in (0,1),
\]
where $i=1,\cdots,d$.  Let $P_{eq}$ be the probability of the event 
$\left\{|\phi(u)-\phi(v)|=|u-v|\right\}$
where the absolute values are taken component-wise.
Then
\begin{eqnarray} \label{prob:preserve}
%pr \left(|\phi(u)-\phi(v)|=|u-v|\right) = 
P_{eq} = 
\left\{
\begin{array}{cc}
p^{2d}, & \phi(t) = \max(0,t), \\[1.5mm]
\left(p^2+(1-p)^2\right)^d, & \phi(t) = |t|.
\end{array}\right.
\end{eqnarray}
\end{proposition}
\begin{proof}
Consider the scalar case $d=1$ with $ u\ne v$.  For ReLU function $\phi(t)=\max(0,t)$,
\begin{equation*}\label{}
\left|\frac{\phi(u)-\phi(v)}{u-v}\right| \left\{
\begin{array}{cc}
 =1, &  u,v \ge 0\\
 <1, &  \mbox{otherwise} 
\end{array}
\right.
\end{equation*}
where the probability for the first case (ratio equal to 1) is $p^2$.
For absolute value $\phi(t)=|t|$,
\begin{equation*}
\left|\frac{\phi(u)-\phi(v)}{u-v}\right| \left\{
\begin{array}{cc}
 =1, &  uv \ge 0 \\
 <1, &  \mbox{otherwise}
\end{array}
\right.
\end{equation*}
where the probability for the first case (ratio equal to 1) is $p^2+(1-p)^2$.

Since all the components are i.i.d., by raising the above probabilities to their $d$-th power, we obtain the corresponding probabilities for the vector case $d>1$.
\end{proof}

The proposition indicates that the probability for ReLU to preserve distances in $\R^d$ is much smaller than that for absolute-value.   In particular, for $p=1/2$ the above two probabilities in \eqref{prob:preserve} become $1/4^d$ and $1/2^d$, respectively; this is, the latter is $2^d$ times larger than the former.  Nevertheless, for either function the probability is large that the diagonal elements $|[\hat{D}_\phi(u,v)]_{ii}|<1$ for at least some indices $i$.  

Figure~\ref{fig:CnG-matrices} illustrates the experimental results for C- and G-matrices. We fix $d$ and increase $L$  instead of keeping $N_w$ constant to ensure that the size of all involved matrices remains unchanged for all $L$.

The left subplot of Figure~\ref{fig:CnG-matrices} shows the spectral norms of matrices $C_L$ and two $G_L$ matrices plotted against depth $L$ ranging from 1 to 1000, for a randomly selected point pair $(x,\bx)$, where ReLU activations are used. Although all three curves exhibit similar patterns of ups and downs, the C-matrix is consistently smaller than the two G-matrices, with a difference of at least one magnitude after $L \ge 200$. This finding suggests that, at least under our experiment setting, C2C should be the main factor responsible for the loss of trainability, rather than the commonly assumed vanishing gradient.
\begin{figure}[htb]
	\begin{center}
		\includegraphics[width=.24\textwidth]{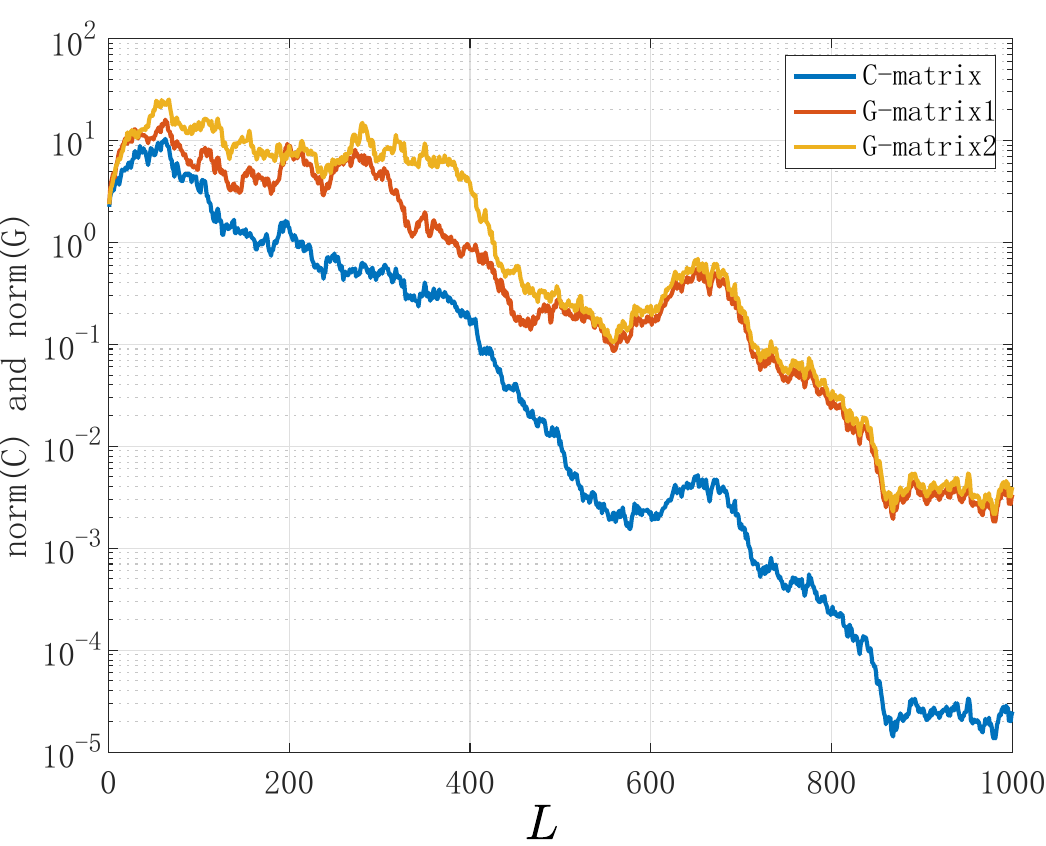}
		\includegraphics[width=.24\textwidth]{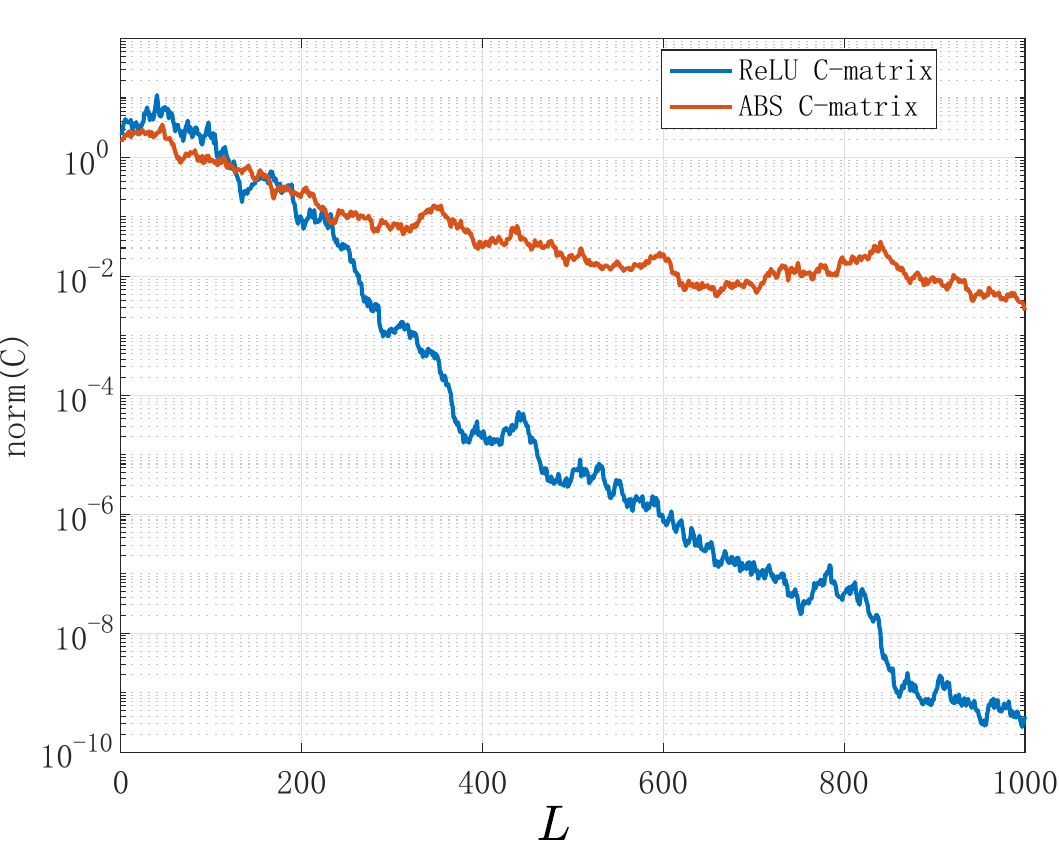}
	\vspace{-.75cm}
	\caption{Norms of C- and G-matrices for ReLU and ABS activations.}
	\label{fig:CnG-matrices}
	\end{center}
	\vspace{-.5cm}
\end{figure}

In the right subplot, the comparison of C-matrices between ReLU and absolute value activations shows that ReLU is more vulnerable to the negative effects of C2C than absolute value activation function, as the C-matrix for ReLU is orders of magnitude smaller than that for the absolute value activation function when depth is large.

\section{Variability vs. Trainability} \label{sec: var vs train}

We have observed that variability of MLPs with a fixed number of parameters changes with depth: it is low at first, then peaks, and then decreases due to C2C. In this section, we present numerical evidence that links the pattern of variability change to the training performance of MLPs. We suggest that variability can serve as a predictive indicator of trainability of MLPs, as $V_3$ correlates highly with the training performance of the corresponding neural networks.

\subsection{Experiment setting}

Our experiments are conducted on a styled synthetic model \emph{checkerboard},  which consists of 6561 mesh points on an $81\times 81$ grid over the square $[-1,1]^2$ in $\mathbb{R}^2$.  These mesh points are divided into two sets, one corresponding to 0-labels and another to 1-labels, so that together they form an 8 by 8 checkerboard blocks, as is shown in Figure \ref{fig:prob4}, where each of the 64 squares contains $81$ grid points and the surrounding edges contain 1377 points.  The blocks take either 0 or 1 (blue or red) label in an alternating pattern, and the surrounding edges all take the 0-label.  In essence, we aim to approximate the piecewise linear, non-smooth function shown in the right plot of Figure~\ref{fig:prob4}.
\begin{figure}[ht]
	%    \vspace{-1em}
	\centering
	\includegraphics[width=.23\textwidth,trim=30 20 30 30, clip]{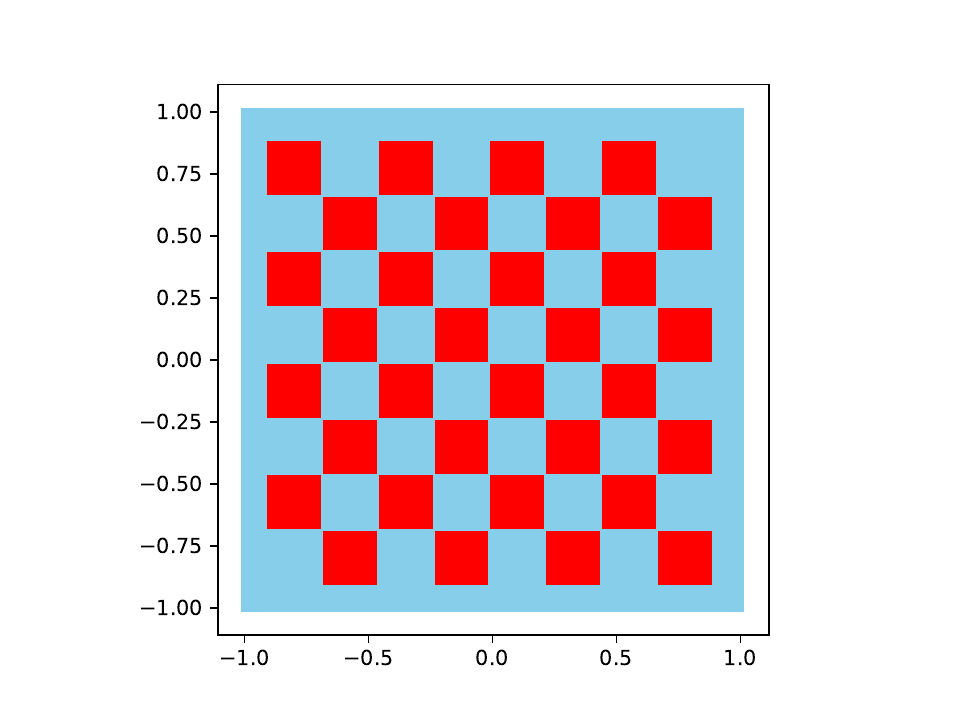}
	\includegraphics[width=.23\textwidth,trim=60 50 60 60, clip]{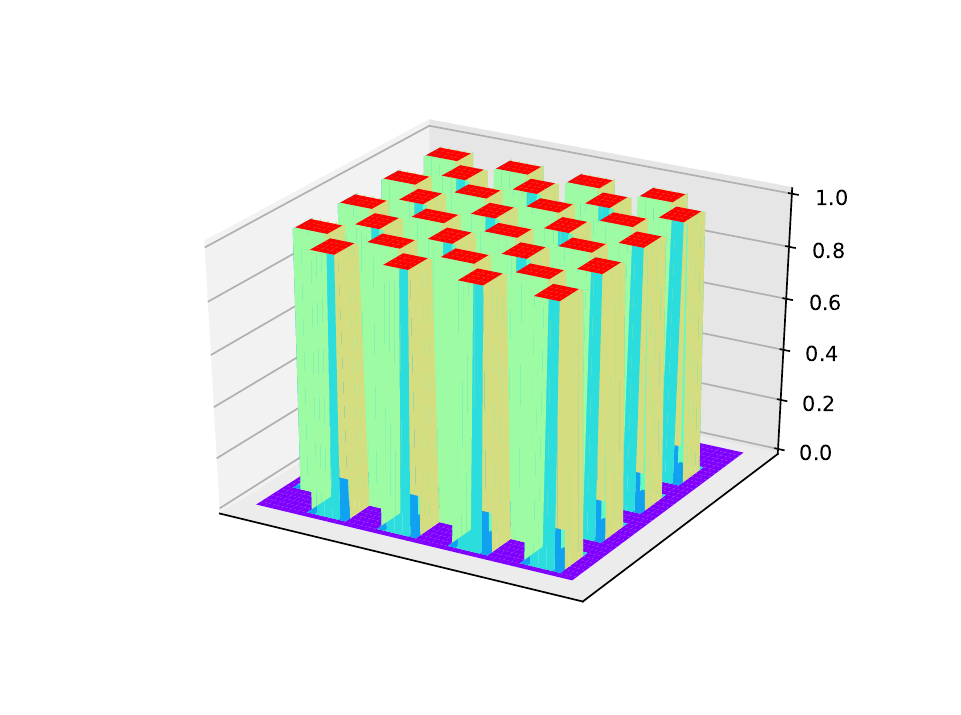}
	%    \vspace{-1em}
	\caption{Checkerboard: left plot for data points with two (colored) classes; right plot for corresponding  binary labels.}
	\label{fig:prob4}
\end{figure}

 \begin{figure*}[th]
 	\vspace{-1.cm}
 	\centering
 	\subfloat[$N_w=1600.$]{
 		\includegraphics[width=.48\textwidth,trim=80 250 50 200, clip]{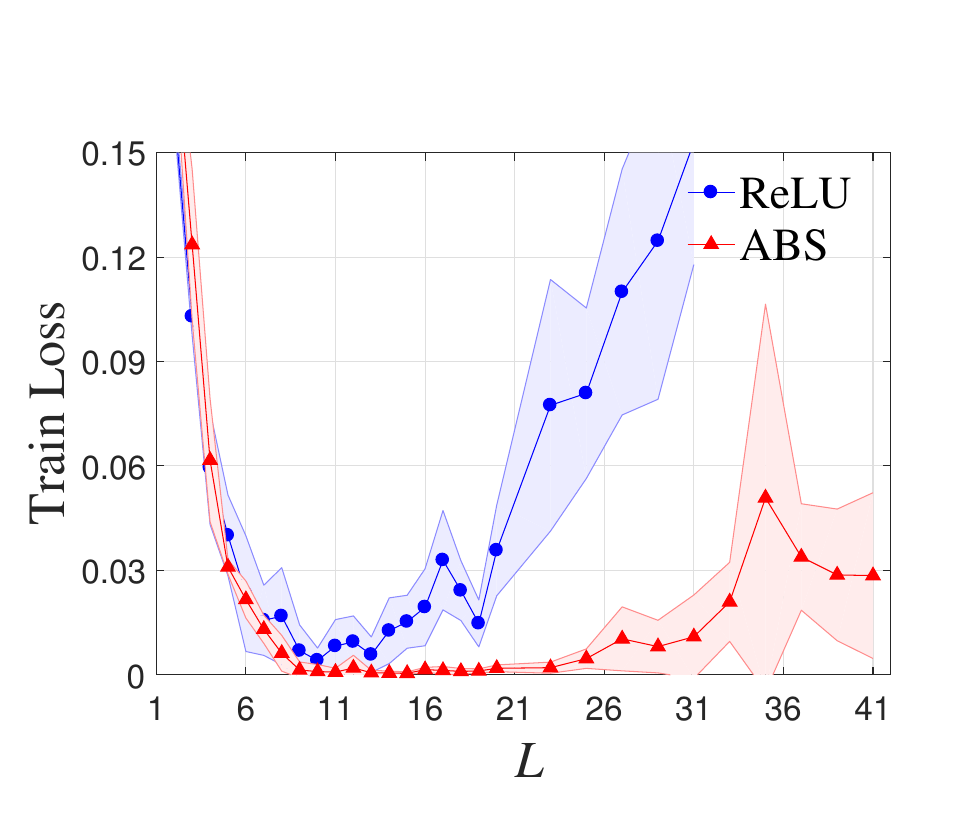}}~~
 	\subfloat[$N_w=3200.$]{
 		\includegraphics[width=.48\textwidth,trim=80 250 50 200, clip]{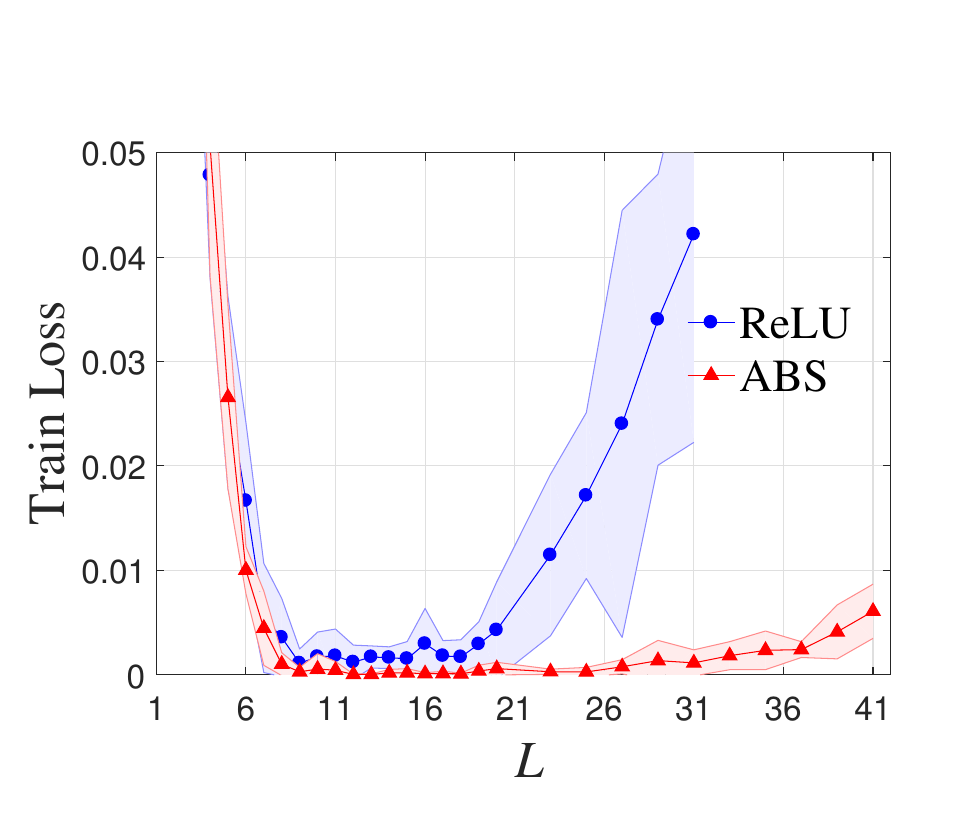}}
 	
 	\caption{Training loss with 1640 samples: ReLU vs ABS for $N_w = 1600$ (left) and $3200$ (right).  
	Each band of curves depicts the mean and variance of 10 random runs.}
 	\label{fig:prob4gd}
 \end{figure*}
We adopt the same MLP models~($N_w=1600, 3200$) used in Figure~\ref{fig:activation-var} with the number of hidden layers $L$ varying from 2 to 31 (with increment 1 up to 20 then increment 2 afterwards) for ReLU and extending to 41 for ABS. We randomly choose $25\%$ data points as the training set with $m=1640$ samples. Denoting the training set by $\{x_i\}_{i=1}^m$, we minimize the least squares loss function,
\begin{equation}\label{min_loss}
    \min_{\bW,\bb}\,%\ell_L(\bW,\bb) \equiv 
    \frac{1}{m}
    \sum_{i=1}^{m} \|F_L(x_i,\bW,\bb)-y_i\|_2^2,
\end{equation}
where each label vector $y_i \in \mathbb{R}^2$ is either $(0,0)^{\T}$ or $(1,1)^{\T}$, representing to the two binary labels. In this set of experiments, we will only examine training loss function values (or train loss, for short), while test loss values are not of concern.

To ensure that the optimization calculation is done sufficiently, we apply the gradient descent method (instead of SGD) with 40000 iterations without a stopping criterion.  For each run, we always try 10 different initial learning rates (step-sizes) as in
$$
\{0.001, ~0.003, ~0.006, ~0.01, ~0.03, ~0.06, ~0.1, ~0.3, ~0.6, ~1.0\}
$$
and then pick the best result for output. During the 40000 iterations, learning rates are reduced by a factor of 5 three times at the junctures corresponding to iterations 20000, 28000, and 36000, respectively. We run each instance with 10 random initial parameter samples and compute the mean value. The result is stable as shown by the std of 10 runs.

\subsection{Computational results}

We show the results of training MLPs of varying depths and parameters ($N_w=1600$ and 3200) in Figure~\ref{fig:prob4gd}. Notably, for Sigmoid, the loss remains around 0.2, irrespective of the model depth ranging from $L=1$ to 10; therefore, we exclude its results from further consideration.  This low trainability associated with Sigmoid activation is evidently explainable by the corresponding low variability, as is shown in Figure~\ref{fig:activation-var}.

Figure~\ref{fig:prob4gd} shows that the variability patterns in Figure~\ref{fig:activation-var} exhibit a striking correlation to the training loss curves in Figure~\ref{fig:prob4gd}.  For ReLU, when $N_w=1600$ and 3200, the peak of variability occurs at $L=12$ and 17 respectively, while the best training performance happens around 11 and 15 respectively. 

For ABS, $V_3$ remains relatively high in the range of $L$ from 8 to 20 and 8 to 30 for $N_w=1600$ and 3200 respectively as Figure~\ref{fig:activation-var} shows. The training loss is relatively close to zero in these two ranges at the same time.

We offer the following interpretations of the experimental results, as pertinent to the relationship between network variability and trainability.
\begin{itemize}

\item Variability in the data space indicates the model's trainability in this experiment. With low variability, models apparently have more local traps, making training difficult.  On the other hand, near or around variability peaks, there appear to exist few or no local traps, as evidenced in Figure~\ref{fig:prob4gd} where the training process seems to reach global optima with few or no exceptions.

\item  ReLU fails to reach near-zero loss values with more than 20 hidden layers, while ABS still succeeds even after the hidden-layer number exceeds 30, confirming that ABS is more effective than ReLU in deeper MLPs.
\end{itemize}

%A well-designed and computable measure of variability, which remains a topic of further research, could serve well as an indicator for experiments with high-dimensional data.

\section{Related Work}

Researchers have extensively studied activation functions as a crucial component of neural network models. The absolute value activation function was sporadically considered in early neural network research, for instance~\cite{batruni1991multilayer, lin1992canonical}, but it has not been widely adopted as a mainstream activation function. Meanwhile, the fact is well-known that the Sigmoid function suffers from gradient-related issues, as is explained in \cite{Goodfellow-et-al-2016}.

Several studies have used the term ``network collapsing" from different perspectives \cite{chen2021exploring, Papyan24652, lu2020dying}. For example, \cite{lu2020dying} studies ``dying ReLU" neural networks that specifically refers to network output collapsing to a constant when ReLU outputs become all zeros. Additionally, \cite{hayou2019impact} observes that if the model is not initialized properly, the outputs will have few variations. 

Some studies mention trainability or related concepts from different perspectives from ours in this paper, such as \cite{xiao2020disentangling,collins2017capacity,sharma2022trainability}.  Analyzing model training inevitably involves gradient exploding and vanishing issues with many studies in this area. For example, \cite{roberts2022principles} has surveyed many articles on this topic.

Many studies, such as \cite{bengio2011expressive,Mhaskar_Liao_Poggio_2017} though too numerous to list even partially, have shown that deep models significantly outperform shallow ones. Our work reports that as the model depth increases, variability first rises and then falls. Meanwhile, practical training performance coincides with the same trend nicely.

A number of existing techniques in deep learning can be interpreted from the viewpoint of enhancing variability of neural networks.  For example, convolutional neural networks (CNN) use far fewer parameters, in comparison to fully connected networks, at each layer, thus greatly increasing activation densities and subsequently variability.  In our view, achieving high activation densities should be considered a significant contributing factor to the great success of CNN.

Since C2C has a close relationship with vanishing gradient, it is not surprising that existing techniques designed to alleviate the latter can also help with the former.  Specifically, since G-matrices are limits of C-matrices, techniques that slow down the size decrease of G-matrices usually also slow down the size decrease of C-matrices.  Such techniques include Residual Networks (or ResNet)~\cite{He_2016_CVPR} and Batch Normalizations \cite{pmlr-v37-ioffe15}.  Another particularly simple technique is to initialize weight matrices by orthogonal matrices \cite{hu2019provable, huang2018orthogonal}.

%People observe that in siamese networks, sometimes all outputs will collapse to a constant due to  similar inputs \cite{chen2020exploring}. \cite{Papyan24652} reports another neural collapse, which begins to happen when training error is close to 0. 

% Expressiveness can be used to explore how the architecture of the model (depth and width)~\cite{HornikApproximation, Cybenko89, NIPS2017_32cbf687, tan2019efficientnet, golubeva2021are} affects performance.

%For example, the authors~\cite{Mhaskar_Liao_Poggio_2017} demonstrate that the number of parameters of deep networks are greatly less than that of shallow ones when approximating compositional functions.

%  Another closely related work is \cite{nguyen2021do}, which measures the similarity of the hidden representations of wide and deep convolutional neural networks. They investigate that when the overall accuracy is similar, wide and deep models still exists distinctive error patterns. However, in our work, we fix the number of parameters to study the effect of width and depth.

\section{Concluding remarks}

In this paper, we introduce a new concept called {\em variability} to understand neural network trainability issues.  In particular, we study multi-layer perceptrons (MLPs) when the number of model parameters is fixed. 
We confirm empirically that variability indeed serves as a key metric for MLP trainability.  It provides explanations for multiple interesting phenomenons related to MLP training, including why deepening depth initially helps training but later hurts it, and why different activation functions behave differently in training.

We studied two quantities to estimate variability: (1) activation ratio which counts a degree of nonlinearity, and (2)  $V_3$ which measures the size of third derivatives.  Despite its simplicity, activation ratio explains well the trend that variability initially increases with the network depth.  The quantity $V_3$ reveals that later on variability starts to decline at the onset of ``Collapse to Constant" (C2C), which is distinct from the well-known vanishing-gradient phenomenon.  Additionally, we find that the absolute-value function is more resistant to C2C than the commonly used ReLU activation function.  Through extensive experiments on a stylized yet non-trivial model problem, we show that variability indeed has a strong positive correlation to MLP trainability.

Finally, we hope that the insight gained from this variability study can become a contributing factor to the considerations of designing new neural network architectures.

\section*{Acknowledgment}

%The preferred spelling of the word ``acknowledgment'' in American English is without an ``e'' after the ``g.'' Use the singular heading even if you have many acknowledgments. Avoid expressions such as ``One of us (S.B.A.) would like to thank ... .'' Instead, write ``F. A. Author thanks ... .'' In most cases, sponsor and financial support acknowledgments are placed in the unnumbered footnote on the first page, not here.

\bibliographystyle{ieeetr}
\bibliography{ref}

% \begin{thebibliography}{34}
% \setcounter{enumiv}{33}

% \bibitem{}S. Azodolmolky {\em et al.}, ``Experimental demonstration of an impairment aware network planning and operation tool for transparent/translucent optical networks,'' {\em J. Lightw. Technol.}, vol. 29, no. 4, pp. 439--448, Sep. 2011.
% \end{thebibliography}

% \begin{IEEEbiography}{A. Author}{\space}was born in
	
% \end{IEEEbiography}

% \begin{IEEEbiography}{B. Author}{\space}was born in

% \end{IEEEbiography}

\end{document}